\documentclass[11pt,letterpaper]{article}
\usepackage[english]{babel}

\usepackage[margin=1in]{geometry}

\usepackage{natbib}

\usepackage[utf8]{inputenc} %
\usepackage[T1]{fontenc}    %
\usepackage{url}            %
\usepackage{booktabs}       %
\usepackage{amsfonts}       %
\usepackage{nicefrac}       %
\usepackage{microtype}      %
\usepackage{xcolor}         %

\usepackage{amsmath}
\usepackage{amssymb}
\usepackage{mathtools}
\usepackage{amsthm}
\usepackage{thmtools, thm-restate}

\theoremstyle{plain}
\newtheorem{theorem}{Theorem}[section]
\newtheorem{proposition}[theorem]{Proposition}
\newtheorem{lemma}[theorem]{Lemma}
\newtheorem{corollary}[theorem]{Corollary}
\theoremstyle{definition}

\theoremstyle{remark}
\newtheorem{remark}[theorem]{Remark}

\numberwithin{equation}{section}

\usepackage[unicode,psdextra,hypertexnames=false]{hyperref}
\hypersetup{
  colorlinks   = true, %
  urlcolor     = blue, %
  linkcolor    = blue, %
  citecolor   = blue, %
  pdfauthor  = {},
  pdfsubject = {},
  pdfkeywords = {}
}
\usepackage{cleveref}
\usepackage{xspace}
\usepackage[vlined,linesnumbered,ruled]{algorithm2e}
\usepackage{float}
\usepackage{wrapfig}
\usepackage{multirow}
\usepackage{colortbl}
\usepackage{hhline}
\usepackage{physics2}
\usephysicsmodule{ab}      %
\usephysicsmodule{diagmat} %
\usephysicsmodule{braket}  %
\usepackage{graphicx}
\usepackage{authblk}

\usepackage{fixdif,derivative}
\usepackage{bm}
\usepackage{bbm}
\usepackage{enumitem}

\renewcommand{\ge}{\geqslant}
\renewcommand{\le}{\leqslant}

\DeclareMathOperator*{\argmax}{arg\,max}

\newcommand{\procname}[1]{\ifmmode\mathop{\text{\textsc{#1}}}\else\textsc{#1}\xspace\fi}
\newcommand{\QMC}{\mathrm{QMC}}            %
\newcommand{\LVQME}{\mathrm{LVQME}}        %
\newcommand{\LVQMEDesignEstimate}{\procname{LVQMEDesignEstimate}} %
\newcommand{\QMCDesignEstimate}{\procname{QMCDesignEstimate}}       %
\newcommand{\LVGElim}{\procname{LV-G-Elim}}        %
\newcommand{\QMCGElim}{\procname{QMC-G-Elim}}        %

\begin{document}

\title{Quantum Multi-Armed Bandits and Linear Bandits: \\
Lower Bounds and Algorithms}
\author[1]{Maoli Liu}
\author[2]{Zhuohua Li\thanks{Corresponding author.}}
\author[1]{John C.S. Lui}
\affil[1]{The Chinese University of Hong Kong}
\affil[2]{Xidian University}
\affil[ ]{ mlliu@cse.cuhk.edu.hk, lizhuohua@xidian.edu.cn, cslui@cse.cuhk.edu.hk}
\date{}

\maketitle

\begin{abstract}
We study quantum multi-armed bandits (QMAB) and quantum linear bandits (QLB), where the learner queries each arm or action through a quantum reward oracle or its inverse.
Prior work gives algorithms over horizon $T$ with regret $O(K\log T)$ for QMAB with $K$ arms and $O(d^2\operatorname{polylog} T)$ for $d$-dimensional QLB.
This leaves open the optimal dependence on $K$ and $T$ and whether the dependence on $d$ can be further improved.
In this work, we prove the first tight minimax regret bound of \(\Theta(K\log(1+T/K))\) for QMAB and the first lower bound of \(\Omega(d\log(1+T/d))\) for finite-action QLB, ruling out regret independent of \(T\).
Our lower bounds rely on a high-confidence single-arm quantum testing lower bound for distinguishing a fixed reward mean from an interval of alternatives.
A bandit-to-testing reduction then lifts it to the QMAB lower bound, while a linear embedding gives the finite-action QLB lower bound.
The matching QMAB upper bound is obtained using a tail bound for the Quantum Monte Carlo (QMC) estimator.
For finite-action QLB, we propose a phased elimination algorithm that combines a low-bias low-variance quantum mean estimator with a small-support $G$-optimal design through a query allocation matched to the design weights.
When the action set has size $\operatorname{poly}(d)$, its regret is nearly linear in $d$ and matches our lower bound up to polylogarithmic factors.
\end{abstract}

\section{Introduction}
\label{sec:intro}

In the stochastic multi-armed bandit (MAB) problem~\citep{robbins1952some,lattimore2020bandit}, a learner is given \(K\) arms, each associated with an unknown reward distribution of mean \(\mu(i)\in[0,1]\).
In each round \(t\in[T]\), the learner selects an arm and observes a reward sampled from its distribution.
The goal is to minimize the cumulative regret \(R(T):=\sum_{t=1}^T(\mu^\star-\mu(a_t))\), where \(\mu^\star\) is the largest mean.
In the stochastic linear bandit problem, actions (also called arms) are vectors \(x\in\mathcal A\subseteq\mathbb R^d\), and the mean reward of \(x\) is \(x^\top\theta\) for an unknown parameter \(\theta\in\mathbb R^d\).
The minimax regret is \(\Theta(\sqrt{KT})\) for MAB~\citep{auer2002nonstochastic,audibert2009minimax} and \(\widetilde\Theta(d\sqrt T)\) for linear bandits with general action  sets~\citep{dani2008stochastic,abbasi2011improved}.
For linear bandits with \(K\) actions, classical algorithms admit regret \(\widetilde O(\sqrt{dT\log K})\)~\citep{lattimore2020bandit}.
Classically, \(\log T\)-type regret appears only as an instance-dependent guarantee: for a fixed instance with constant gaps, the optimal regret scales as \(\Theta(\log T)\)~\citep{lai1985asymptotically}.

\citet{wan2023quantum} introduced the quantum bandit model we study, in which each arm is equipped with a quantum reward oracle, a unitary that encodes its reward distribution.
In each round, the learner selects an arm and applies the corresponding oracle or its inverse at most once, instead of observing a classical sample.
With such oracle access, quantum mean estimation methods such as Quantum Monte Carlo (QMC) estimate a mean to accuracy \(\epsilon\) using \(O(1/\epsilon)\) oracle queries~\citep{montanaro2015quantum}, rather than the classical \(\Theta(1/\epsilon^2)\) samples.
Building on this speedup, \citet{wan2023quantum} obtained regret \(O(K\log T)\) for quantum multi-armed bandits (QMAB) under bounded rewards.
For quantum linear bandits (QLB), their QLinUCB algorithm applies to general action sets and gives an explicit \(O(d^2\operatorname{polylog}T)\) regret bound.

These upper bounds leave two questions open.
First, the optimal dependence on \(K\) and \(T\) for QMAB remains unknown. 
In particular, \citet{wan2023quantum} explicitly asked whether regret independent of \(T\) is achievable or can be ruled out by a lower bound.
Second, the optimal dependence on \(d\) for finite-action QLB is unclear.
Finite action sets form a well-studied setting in classical linear bandits, with regret bounds that explicitly use the fact that the action set is finite.
In contrast, their \(O(d^2\operatorname{polylog} T)\) bound applies to finite action sets only as a special case of the result for general action sets without exploiting this fact.

We address both questions.
We establish the first tight minimax regret bound of \(\Theta(K\log(1+T/K))\) for QMAB and prove an \(\Omega(d\log(1+T/d))\) lower bound for finite-action QLB with \(d\) actions.
Complementing this lower bound, we propose \LVGElim, an algorithm whose regret is nearly linear in \(d\) under bounded rewards when \(K=\operatorname{poly}(d)\).
To prove the regret lower bounds, we first study the testing problem of distinguishing a fixed reward mean from an interval of alternatives.
For tests with a fixed query budget, we use the polynomial method and a Remez-type inequality for trigonometric polynomials to prove a query lower bound under asymmetric error requirements.
A truncation argument then uses this asymmetry to obtain a lower bound on the expected number of queries.
By converting bandit policies into such tests, we derive the regret lower bounds.
The QMAB upper bound follows from a refined analysis using a tail bound for the QMC estimator, while the finite-action QLB upper bound is obtained by combining a small-support \(G\)-optimal design with a low-bias low-variance quantum mean estimator through a query allocation matched to the design weights.
We state the results next and then give a technical overview.

\subsection{Our Results}

Our first result establishes matching upper and lower bounds for QMAB, resolving the question of \citet{wan2023quantum} by showing that worst-case regret must grow logarithmically with \(T\).

\begin{theorem}[QMAB minimax regret; informal version of Theorems~\ref{thm:qmab-lb} and~\ref{thm:qmab-upper-bound}]
\label{thm:informal-qmab-lb}
For every \(K\ge2\) and \(T\ge1\), the minimax expected regret of quantum multi-armed bandits with \(K\) arms is \(\Theta(K\log(1+T/K))\).
\end{theorem}

The hard instances used for the QMAB lower bound can be realized as a linear bandit with \(d\) actions, yielding the following finite-action QLB lower bound.

\begin{theorem}[Finite-action QLB lower bound; informal version of Theorem~\ref{thm:qlb-lb}]
\label{thm:informal-qlb-lb}
For every \(d\ge2\) and \(T\ge 1\), there is an action set \(\mathcal A\subseteq\mathbb R^d\) with \(|\mathcal A|=d\) such that any quantum linear bandit policy \(\pi\) over \(T\) rounds suffers expected regret \(\Omega(d\log(1+T/d))\) on some instance with action set \(\mathcal A\).
\end{theorem}

Our algorithm \LVGElim{} nearly matches the linear bandit lower bound when \(K=\operatorname{poly}(d)\).

\begin{theorem}[Regret of \LVGElim; informal version of Theorem~\ref{thm:lv-g-elim-regret}]
\label{thm:informal-lv-g-elim}
For finite-action quantum linear bandits with \(K\) actions, \LVGElim has expected regret \(O\big(d\log(KT)\cdot\operatorname{polylog}(d,T)\big)\).
\end{theorem}

\begin{table}[t]
\centering
\caption{Regret bounds for multi-armed bandits and linear bandits.}
\label{tab:bounds}
\resizebox{\textwidth}{!}{
\begin{tabular}{@{}llccc}
\toprule
 & Bound & MAB & Linear, general \(\mathcal A\) & Linear, finite \(\mathcal A\) \\
\midrule
\multirow{2}{*}{Classical}
 & UB & \(O(\sqrt{KT})\) & \(O(d\sqrt T\log T)\) & \(O(\sqrt{dT\log K})\) \\
 & LB & \(\Omega(\sqrt{KT})\) & \(\Omega(d\sqrt T)\) & \(\Omega(\sqrt{dT})\) \\
\midrule
\multirow{4}{*}{Quantum}
 & UB~\citep{wan2023quantum} & \(O(K\log T)\) & \(O(d^2\,\mathrm{polylog}\,T)\) & \(O(d^2\,\mathrm{polylog}\,T)\) \\
  & UB (Appendix~\ref{app:hikima-linear})\(^{\dagger}\)
 & --- & \(O(d^{3/2}\operatorname{polylog} T)\) & \(O(d^{3/2}\operatorname{polylog} T)\) \\
 & \cellcolor{gray!15}UB (Sections~\ref{sec:qmab-upper-bound}--\ref{sec:finite-action-qlb}) & \cellcolor{gray!15}\({O}(K\log(1+T/K))\) & \cellcolor{gray!15}--- & \cellcolor{gray!15}\(O(d\log T\log(KT))\,^{\ddagger}\) \\
 & \cellcolor{gray!15}LB (Section~\ref{sec:lower-bound}) & \cellcolor{gray!15}\(\Omega(K\log(1+T/K))\) & \cellcolor{gray!15}\(\Omega(d\log(1+T/d))\) & \cellcolor{gray!15}\(\Omega(d\log(1+T/d))\) \\
\specialrule{\heavyrulewidth}{0pt}{0pt} 
\end{tabular}
}
\smallskip
\begin{minipage}{\linewidth}
\footnotesize
The quantum bounds are for the model with quantum reward oracles (Section~\ref{sec:prelim}), under the bounded-reward assumption. \citet{wan2023quantum} also give bounds under a bounded-variance assumption. 
UB and LB denote upper and lower bounds. 
\(^{\dagger}\)~Obtained by specializing the analysis of \citet{hikima2024quantum} to the linear kernel.
\(^{\ddagger}\)~Up to \(\log(dT)\log\log(dT)\) factors.
\end{minipage}
\end{table}

Table~\ref{tab:bounds} compares our results with existing regret bounds for classical and quantum bandits.
For classical MAB, the upper bound is due to \citet{audibert2009minimax} and the lower bound is from \citet{auer2002nonstochastic}.
For linear bandits with general action sets, the upper bound is due to \citet{abbasi2011improved} and the lower bound is from \citet{dani2008stochastic}.
For finite-action linear bandits, the upper bound is from \citet[Chapter~22]{lattimore2020bandit}, while the lower bound follows from the standard embedding of \(d\)-armed bandits.

Our QMAB upper and lower bounds match.
Our QLB lower bound is proved with Bernoulli rewards and a finite action set with \(K=d\) actions, and hence also applies to general action sets (see Remark~\ref{rem:hard-instances}).
For finite action sets with \(K=\operatorname{poly}(d)\), our upper and lower bounds nearly match, while for general action sets the optimal dimension dependence remains open.

\subsection{Technical Overview}
We first explain the lower bound argument, then outline the algorithms and analysis used to obtain the upper bounds.
\paragraph{Lower bounds.}
Neither of the two standard classical lower bound mechanisms directly yields our quantum lower bounds.
Instead, as we explain below, the hardness comes from testing against an interval of alternatives rather than a fixed point.
Classical minimax lower bounds rest on a two-point change-of-measure argument: two reward distributions whose means differ by \(\epsilon\) can have KL divergence of order \(\epsilon^2\), so distinguishing them even with constant error probability requires \(\Omega(1/\epsilon^2)\) samples. 
Balancing this testing cost with the regret incurred by a gap \(\epsilon\) yields the familiar \(\Omega(\sqrt{KT})\) rate. 
However, quantum mean estimation distinguishes two such distributions using only \(O(1/\epsilon)\) queries, weakening this mechanism.
The second classical mechanism comes from instance-dependent lower bounds, such as the Lai--Robbins bound~\citep{lai1985asymptotically}.
These bounds operate in a high-confidence regime: driving the probability of confusing a suboptimal arm with the optimal one down to \(\delta\) gives the sample cost a \(\log(1/\delta)\) factor. 
Taking \(\delta\) of order \(1/T\) produces the classical \(\log T\) term.
In the quantum model, this source of difficulty is absent for two-point comparisons: for a fixed pair of distinct oracle unitaries, the cost of distinguishing them is bounded by a quantity determined by the pair, with no \(\log(1/\delta)\) growth~\citep{acin2001statistical,duan2007entanglement}.
Our regret lower bounds, however, require exactly this \(\log(1/\delta)\) growth, so it must be recovered from a source other than two-point comparisons.
This is what the interval of alternatives provides.

The hard instances have a needle-in-a-haystack structure.
In the baseline instance, one reference arm has a higher mean than the candidate arms, which share the same mean.
In each alternative instance, one of these candidate arms has its mean elevated to an unknown value in an interval above the reference arm's mean.
For this family, low regret under the baseline instance limits how often the policy can select the candidate arms, while low regret on the alternatives requires it to select the elevated arm sufficiently often.
This converts regret minimization into single-arm tests that distinguish, with error probability \(\delta\), a baseline mean from an interval of elevated means.
The interval is essential because a single fixed alternative would not produce the required \(\log(1/\delta)\) growth.

For a policy with worst-case regret \(M_T\), these single-arm tests have error at most \(\delta=O(M_T/T)\) under both the baseline and every alternative, and make a random number of queries.
To bound their expected number of queries under the baseline, we truncate each test at a fixed query budget and output alternative upon truncation.
With a suitable budget, the baseline error increases only to a constant, while the error under every alternative remains at most \(\delta\). 
It therefore suffices to prove a fixed-query lower bound under these asymmetric error requirements.
For a \(q\)-query quantum test, the probability of declaring the mean to be the baseline value is a trigonometric polynomial of degree \(O(q)\) in the amplitude parameter.
A Remez-type inequality shows that if such a polynomial is at least a positive constant at the baseline while remaining uniformly at most \(\delta\) on an interval separated from it, then its degree is \(\Omega(\log(1/\delta))\).
Combined with the degree bound \(O(q)\), this forces \(q=\Omega(\log(1/\delta))\).
Since the truncation budget can be chosen proportional to the expected number of queries under the baseline, that expectation is also \(\Omega(\log(1/\delta))\).
Applying this expected-query lower bound across the candidate arms gives \(M_T \gtrsim K\log(T/M_T)\), which yields the QMAB lower bound \(\Omega(K\log(1+T/K))\).
Realizing the same hard family as a \(d\)-action linear bandit gives \(\Omega(d\log(1+T/d))\) for finite-action QLB.

\paragraph{QMAB upper bound.}
The \(\mathrm{QUCB}_1\) algorithm of~\citet{wan2023quantum} achieves an \(O(K\log T)\) regret bound for QMAB.
We obtain the matching \(O(K\log(1+T/K))\) expected regret bound by adjusting its parameters and refining its expected regret analysis.
Their analysis uses a union bound over \(O(K\log(1+T/K))\) QMC calls, each with failure probability at most \(\delta\).
It then bounds the regret by \(T\) whenever any estimate differs from the true mean by more than its corresponding radius.
When every estimate has error at most its corresponding radius,
their analysis gives an \(O(K\log(1/\delta))\) regret bound.
Taking \(\delta=1/T\) then gives an \(O(K\log T)\)
expected regret bound.
However, an estimate with error larger than its confidence radius need not lead to regret of order \(T\).
We instead use the QMC tail bound to control the expected regret in different types of stages directly.
We first bound the regret incurred while the selected arm's radius is large relative to its gap.
For the remaining stages, the QMC tail bound controls how far an optimal arm's upper confidence bound falls below its true mean in expectation, as well as the probability of a large overestimate of a suboptimal arm.
This yields \(O(K\log(1/\delta)+T\delta)\), without an additional factor for the number of QMC calls in the \(T\delta\) term.
Optimizing \(\delta\) gives the matching \(O(K\log(1+T/K))\) expected regret bound.

\paragraph{Finite-action QLB upper bound.}
The QLinUCB algorithm of \citet{wan2023quantum} applies to general action sets and therefore also to the finite action sets considered here.
It is a quantum analogue of the OFUL (Optimism in the Face of Uncertainty for Linear bandits) confidence-ellipsoid approach~\citep{abbasi2011improved}, and its analysis yields an \(O(d^2\mathrm{polylog}\,T)\) regret bound.
Our approach instead uses phased elimination with \(G\)-optimal design~\citep{lattimore2020bandit}, a scheme specific to finite action sets.
In each phase, a design selects support actions to query, and the resulting mean estimates are extrapolated to all active actions with error amplification controlled by the design.
The quantum setting changes the algorithmic design and analysis in two respects.

First, the support size moves from an additive overhead to a leading term in the regret bound.
In the classical analysis, estimating to accuracy \(\epsilon\) costs \(\epsilon^{-2}\) samples in the leading term, and the support size appears only as an additive rounding overhead.
By contrast, quantum mean estimation has leading cost \(\epsilon^{-1}\), and each support action requires its own estimator call with a constant minimum number of queries.
Since each query in a phase incurs regret proportional to the phase accuracy, the support size enters the leading term of the regret bound.
We therefore use a small-support approximate \(G\)-optimal design with support size \(O(r)\), where \(r\) is the dimension of the span of the active set.

Second, the choice of estimator determines how errors aggregate when the support estimates are extrapolated.
A direct high-probability estimator such as QMC controls worst-case absolute errors on the support, and extrapolation through the design can amplify such errors by a factor of \(\sqrt r\).
Thus achieving phase accuracy \(\epsilon\) requires support accuracy \(\rho\asymp\epsilon/\sqrt r\), giving a \(d^{3/2}\) dependence in the regret bound.
Our algorithm instead uses the low-bias low-variance quantum mean estimator (LVQME) of \citet{cornelissen2023sublinear}: with a query allocation matched to the design weights, independent fluctuations aggregate through variance, and the \(\sqrt r\) loss disappears.
However, LVQME uses a random number of queries and provides only a bound on the expected query count. 
We therefore impose a shared threshold on the total queries.
This enforces a deterministic budget while retaining accuracy with constant probability for each fixed active action.
Median amplification then gives simultaneous high-probability accuracy over the active set and introduces a \(\log K\) factor.
When \(K=\operatorname{poly}(d)\), the regret is nearly linear in \(d\), matching the lower bound up to polylogarithmic factors.

\subsection{Related Work}
\paragraph{Quantum bandits.}
The quantum bandit model we study was introduced by~\citet{wan2023quantum}, who obtained regret \(O(K\log T)\) for multi-armed bandits and \(O(d^2\,\mathrm{polylog}\,T)\) for linear bandits under bounded rewards, with counterparts under a bounded-variance assumption that carry additional logarithmic factors.
\citet{dai2023quantum} extended the model to kernelized rewards, but their analysis relies on a sub-Gaussianity assumption on the QMC error that, as \citet{hikima2024quantum} point out, is not guaranteed by the QMC method.
\citet{hikima2024quantum} give an algorithm for quantum kernelized bandits whose analysis avoids this assumption.
Although they do not state a separate linear bandit result, specializing their confidence bound and determinant-based stage-count argument to the linear kernel yields regret \(O(d^{3/2}\,\mathrm{polylog}\,T)\) for general action sets (see Appendix~\ref{app:hikima-linear}).
In the same oracle model, \citet{wang2025best} study best arm identification, and further consider oracles that address up to \(m\) arms in superposition, gaining an additional \(\sqrt m\) factor with matching query lower bounds.
The model has also been extended to heavy-tailed rewards~\citep{wu2023quantum}, Lipschitz bandits~\citep{yi2026quantum}, and bandits with knapsacks~\citep{su2025quantum}.
These works do not establish minimax regret lower bounds in this model.

A parallel line of work studies best arm identification with coherent access to the arms~\citep{casale2020quantum,wang2021quantum}.
In this stronger model, \citet{wang2021quantum} obtain a quadratic speedup with a matching query lower bound.
At the other end, \citet{buchholz2025multi} show that when the reward randomness cannot be accessed coherently, the quantum advantage disappears entirely.
Both the achievable speedup and the ultimate limits thus depend on the access model, and our lower bounds characterize these limits for the model of \citet{wan2023quantum}.

In a different direction, \citet{lumbreras2022multi} study bandits whose arms are observables applied to an unknown quantum state and prove \(\Omega(\sqrt T)\) regret lower bounds, which do not transfer to our setting since the model and the source of randomness differ.
Relatedly, \citet{lumbreras2024linear} show that classical linear bandits with noise vanishing near the optimal action admit \(\mathrm{polylog}\,(T)\) minimax regret, with applications to learning pure quantum states~\citep{lumbreras2026learning}.
This gives a route to logarithmic regret different from the query speedups studied in this paper.
\citet{huang2026balancing} study Gaussian process bandit optimization with quantum kernels, where the learner observes noisy rewards rather than accessing quantum reward oracles.

\paragraph{Classical linear bandits and experimental design.}
For general action sets, confidence-ellipsoid algorithms achieve \(\widetilde O(d\sqrt T)\) regret~\citep{dani2008stochastic,abbasi2011improved}.
For finite action sets, phased elimination with \(G\)-optimal design achieves \(O(\sqrt{dT\log K})\) regret~\citep{lattimore2020learning,lattimore2020bandit}.
Our algorithm follows this finite-action scheme, with the sampling step replaced by quantum estimation, which in turn constrains the design.
For the design itself, the Kiefer--Wolfowitz theorem~\citep{kiefer1960equivalence} guarantees a \(G\)-optimal design, \citet{lattimore2020learning} obtain one with support size \(O(d\log\log d)\), and we use the rounding method of \citet{allen2021near} to obtain support \(O(r)\) on an action set of rank \(r\).
In the classical setting, small support is mainly a computational convenience.
By contrast, it is essential in our quantum algorithm, since the query allocation must place a minimum number of queries on every support point.
For contextual linear bandits with action sets that may change over time, SupLinUCB~\citep{chu2011contextual} attains \(\widetilde O(\sqrt{dT\log K})\) regret.
In this setting, \citet{li2019nearly} prove a lower bound of order \(\Omega(\sqrt{dT\log(T/d)\log K})\).
Their upper bound matches it up to iterated logarithmic factors and the replacement of \(\log(T/d)\) by \(\log T\).
Design-based allocation also plays a central role in pure exploration~\citep{soare2014best,fiez2019sequential}.

\paragraph{Quantum query lower bounds.}
Quantum query lower bounds often proceed by representing acceptance probabilities as low-degree polynomials~\citep{beals2001quantum}.
Closest to our fixed-query testing lemma is \citet{mande2026tight}, who prove an \(\Omega(\log(1/\delta))\)-query lower bound for achieving error probability \(\delta\) in phase estimation.
Their proof likewise represents acceptance probabilities as trigonometric polynomials.
In their problem, the alternatives cover all phases outside a small window, so a growth inequality on large sets~\citep[Theorem~5.1.2]{borwein1995polynomials} suffices.
Our alternatives instead form a short interval, and this requires a Remez-type inequality on small sets~\citep{ganzburg2012remez}.
Moreover, a bandit policy spends a random number of queries on each arm, so our fixed-query bound must be extended to expected query counts through a truncation argument that exploits its asymmetric error requirements.
Finally, the reduction from bandit regret to point-versus-interval testing has no analogue in phase estimation, and is what turns the testing lower bound into a regret lower bound.

\section{Preliminaries}
\label{sec:prelim}

\paragraph{Notation.}
For a positive integer \(n\), let \([n]:=\{1,\ldots,n\}\).
We use \(\|\cdot\|_2\) for the Euclidean norm, \(e_1,\ldots,e_d\) for the standard basis of \(\mathbb R^d\), and \(\operatorname{span}(\mathcal X)\) for the linear span of a set \(\mathcal X\).
For a matrix \(M\), \(\operatorname{Range}(M)\) denotes its range and \(M^+\) its Moore--Penrose pseudoinverse.
All logarithms are natural unless the base is written explicitly.
In informal statements, \(a\gtrsim b\) means \(a\ge Cb\) for a universal constant \(C>0\). 

\paragraph{Quantum reward oracles.}
Let \(\Omega\) be a finite sample space equipped with a probability measure \(P\), and let \(y:\Omega\to[0,1]\) be a random variable representing the reward, with mean \(\mu:=\mathbb E_{\omega\sim P}[y(\omega)]\).
The \emph{quantum reward oracle} associated with \((P,y)\) is a unitary \(\mathcal O\) acting as
\begin{equation}
\label{eq:oracle}
    \mathcal O: |0\rangle
    \mapsto
    \sum_{\omega\in\Omega}\sqrt{P(\omega)}\,|\omega\rangle |y(\omega)\rangle,
\end{equation}
where \(|0\rangle\) denotes the all-zero state of both registers, an outcome register for \(\omega\) and a reward register for \(y(\omega)\).
Measuring the reward register of the output state \(\mathcal O|0\rangle\) in the computational basis yields a classical sample of \(y\), so the oracle \(\mathcal O\) generalizes classical reward feedback. 

\paragraph{Quantum multi-armed bandits.}
A \emph{quantum multi-armed bandit} (QMAB) instance \(\nu\) consists of \(K\) arms.
Each arm \(i\in[K]\) is associated with an unknown reward distribution \((P_i,y_i)\), where \(P_i\) is a probability measure on a finite sample space \(\Omega_i\) and \(y_i:\Omega_i\to[0,1]\) is the reward.
The mean reward of arm \(i\) is denoted by \(\mu(i)\).
In the quantum bandit model of \citet{wan2023quantum}, access to each arm \(i\) is only through a quantum reward oracle \(\mathcal O_i\) of the form~\eqref{eq:oracle} that encodes \((P_i,y_i)\).

The learner interacts with the instance over a \emph{horizon} of \(T\) rounds, maintaining a quantum circuit throughout.
In each round \(t\in[T]\), the learner selects an arm \(a_t\in[K]\) and applies either \(\mathcal O_{a_t}\) or \(\mathcal O_{a_t}^\dagger\) at most once.
Between consecutive rounds, the learner may place arbitrary unitaries that do not depend on the unknown reward distributions, and it may perform measurements at any round.
We call the learner's strategy a \emph{policy}, denoted \(\pi\).

Let \(i^\star\in\argmax_{i\in[K]}\mu(i)\) be an optimal arm, and let \(\Delta_i:=\mu(i^\star)-\mu(i)\) be the reward gap of arm \(i\).
The learner's objective is to minimize the cumulative regret, the total gap from always selecting an optimal arm:
\[
    R(T):=\sum_{t=1}^T\big(\mu(i^\star)-\mu(a_t)\big).
\]
For a policy \(\pi\) and an instance \(\nu\), we write \(R_T^\pi(\nu):=\mathbb E_\nu[R(T)]\) for the expected regret, and use \(\Pr_\nu\) and \(\mathbb E_\nu\) for probabilities and expectations when \(\pi\) runs on \(\nu\).

\paragraph{Quantum linear bandits.}
A \emph{quantum linear bandit} (QLB) instance consists of a known action set \(\mathcal A\subseteq\mathbb R^d\) and an unknown parameter \(\theta\in\mathbb R^d\).
The mean reward of an action \(x\in\mathcal A\) is \(\mu(x):=x^\top\theta\).
We assume \(\|x\|_2\le 1\) for all \(x\in\mathcal A\), \(\|\theta\|_2\le 1\), and \(\mu(x)\in[0,1]\) for all \(x\in\mathcal A\).
As in the QMAB model, each action \(x\in\mathcal A\) is accessed through a quantum reward oracle \(\mathcal O_x\) of the form~\eqref{eq:oracle} whose reward mean is \(\mu(x)\).
In this paper, we focus on finite action sets and write \(K:=|\mathcal A|\) for the number of actions.

The interaction protocol is the same as in the multi-armed case, with arms replaced by actions. The cumulative regret is defined in the same way, with \(i^\star\) replaced by an optimal action \(x^\star\in\argmax_{x\in\mathcal A}\mu(x)\).
For a policy \(\pi\), we write \(R_T^\pi(\theta)\) for the expected regret on the instance determined by the parameter \(\theta\).

\paragraph{Quantum mean estimation.}
Our algorithms estimate the mean rewards of actions with two quantum subroutines.
The first is the standard Quantum Monte Carlo (QMC) estimator, which guarantees a small error with high probability and underlies the algorithms of \citet{wan2023quantum}.
The second is the nondestructive unbiased amplitude estimator of \citet{cornelissen2023sublinear}, which controls the bias and the variance.
We first restate the guarantee for the QMC estimator.

\begin{lemma}[Quantum Monte Carlo {\citep[Theorem~2.3]{montanaro2015quantum}}]
\label{lem:qmc}
There exists a constant \(C_1>1\) and a quantum algorithm \(\QMC(\mathcal O,\epsilon,\delta)\) such that, for any \(\epsilon\in(0,1]\) and \(\delta\in(0,1)\), the algorithm outputs an estimate \(\widehat\mu\) of \(\mu\) satisfying \(\Pr\big[\lvert\widehat\mu-\mu\rvert>\epsilon\big]\le\delta\).
It uses at most \(\frac{C_1}{\epsilon}\log\frac{1}{\delta}\) queries
to \(\mathcal O\) and \(\mathcal O^\dagger\).
\end{lemma}

The QMC estimator also satisfies the following tail bound.
For completeness, we provide a proof based on
\citet{brassard2000quantum,nagaj2009fast}
in Appendix~\ref{app:qmc-tail}.

\begin{lemma}[Tail bound for Quantum Monte Carlo]
\label{lem:qmc-tail}
For the \(\QMC\) algorithm in Lemma~\ref{lem:qmc}, there exists a universal constant \(C_2>1\) such that, for any \(\epsilon\in(0,1]\) and \(\delta\in(0,1/2]\), the algorithm \(\QMC(\mathcal O,\epsilon,\delta)\) outputs an estimate \(\widehat\mu\) of \(\mu\) satisfying, for all \(\gamma\ge\epsilon\),
\[
    \Pr\big[\lvert\widehat\mu-\mu\rvert\ge \gamma\big]
    \le \delta\left(\frac{\epsilon}{\gamma}\right)^{\frac12\log\frac{1}{\delta}}.
\]
It uses at most \(\frac{C_2}{\epsilon}\log\frac{1}{\delta}\) queries to \(\mathcal O\) and \(\mathcal O^\dagger\).
\end{lemma}

We next recall the guarantee for the nondestructive unbiased amplitude estimator.

\begin{lemma}[Nondestructive unbiased amplitude estimation {\citep[Theorem~2.4]{cornelissen2023sublinear}}]
\label{lem:ndub-amplitude}
Let \(|\psi\rangle\) be a quantum state and let \(\Pi\) be a projection operator with \(p=\|\Pi|\psi\rangle\|^2\).
Given \(t\ge4\) and \(\varepsilon\in(0,1)\), the algorithm \(\procname{NdUb-Amplitude}(|\psi\rangle,\Pi,t,\varepsilon)\) outputs an estimate \(\widetilde p\in[-2\pi,2\pi]\) such that
\[
    \bigl\lvert\mathbb E[\widetilde p]-p\bigr\rvert\le\varepsilon,
    \qquad
    \operatorname{Var}(\widetilde p)\le\frac{91\,p}{t^2}+\varepsilon .
\]
The algorithm needs one copy of \(|\psi\rangle\), which is restored at the end of the computation with probability at least \(1-\varepsilon\), and \(O(t\log\log(t)\log(t/\varepsilon))\) applications in expectation of the reflection operators \(I-2|\psi\rangle\langle\psi|\) and \(I-2\Pi\).\footnote{In the construction underlying \citet[Theorem~2.4]{cornelissen2023sublinear}, the reflection primitives are applied also in controlled form (see their Lemmas~A.1, A.5, and~B.1), and such applications are included in the reflection count of Theorem~2.4.
The amplification and conversion unitaries there are implemented from the same reflection primitives (their Lemmas~A.3--A.4).}
\end{lemma}

Specializing this estimator to reward oracles yields the following mean estimator, which is the form we use.

\begin{lemma}[Low-bias low-variance quantum mean estimation;
from {\citealp[Theorem~2.4]{cornelissen2023sublinear}}]
\label{lem:lvqme}
There exist constants \(C_Q,C_V>0\) and a quantum algorithm \(\LVQME(\mathcal O,t,\zeta)\) such that, for any \(t\ge4\) and \(\zeta\in(0,1)\), the algorithm outputs an estimate \(\widehat\mu\) of \(\mu\) satisfying
\[
    \bigl\lvert\mathbb E[\widehat\mu]-\mu\bigr\rvert\le\zeta,
    \qquad
    \operatorname{Var}(\widehat\mu)\le\frac{C_V}{t^2}+\zeta.
\]
Its expected number of queries to \(\mathcal O\) and \(\mathcal O^\dagger\) is at most \(C_Q\,t\,\log\log(t)\log\frac{t}{\zeta}\).\footnote{Only the expectation of the query count is bounded. The design estimation procedure in Section~\ref{sec:finite-action-qlb} therefore enforces a deterministic query threshold.}
\end{lemma}

\begin{proof}
We obtain the lemma by specializing Lemma~\ref{lem:ndub-amplitude}.
Add one ancilla qubit, and let \(W\) be the unitary that rotates the ancilla according to the reward value, \(W:|y\rangle|0\rangle\mapsto|y\rangle\big(\sqrt{1-y}\,|0\rangle+\sqrt{y}\,|1\rangle\big)\) for \(y\in[0,1]\), extended arbitrarily to a full unitary.
Set \(V:=W(\mathcal O\otimes I)\), and let \(|\psi\rangle:=V|0\rangle\), where \(|0\rangle\) denotes the all-zero initial state including the ancilla.
Let \(\Pi:=I\otimes|1\rangle\langle1|\) be the projection onto the ancilla state \(|1\rangle\).
Then \(\|\Pi|\psi\rangle\|^2=\sum_{\omega\in\Omega}P(\omega)\,y(\omega)=\mu\), so Lemma~\ref{lem:ndub-amplitude} applied to \((|\psi\rangle,\Pi)\) directly yields an estimate of \(\mu\).

The algorithm \(\LVQME(\mathcal O,t,\zeta)\) prepares \(|\psi\rangle\), runs \(\procname{NdUb-Amplitude}(|\psi\rangle,\Pi,t,\zeta)\), and returns its output as \(\widehat\mu\).
Since \(t\ge4\) and \(\zeta\in(0,1)\), Lemma~\ref{lem:ndub-amplitude} gives \(\lvert\mathbb E[\widehat\mu]-\mu\rvert\le\zeta\) and \(\operatorname{Var}(\widehat\mu)\le91\mu/t^2+\zeta\le91/t^2+\zeta\), using \(\mu\le1\), so the variance bound holds with \(C_V:=91\).
Note that although the estimate \(\widehat\mu\) may lie outside \([0,1]\), the analysis in Section~\ref{sec:finite-action-qlb} uses only its bias and variance guarantees.

It remains to count oracle queries.
Preparing \(|\psi\rangle\) uses one query to \(\mathcal O\).
Write \(R_\psi:=I-2|\psi\rangle\langle\psi|\), \(R_\Pi:=I-2\Pi\), and \(R_0:=I-2|0\rangle\langle0|\).
Since \(|\psi\rangle=V|0\rangle\), the identity \(R_\psi=VR_0V^\dagger\) implements \(R_\psi\) using one query each to \(\mathcal O\) and \(\mathcal O^\dagger\), while \(R_\Pi=I\otimes(I-2|1\rangle\langle1|)\) is independent of the oracle.

The same accounting applies to controlled reflections.
For any unitary \(D\) that is independent of the oracle and any number \(j\) of control qubits, \(\mathrm c^j\text{-}(VDV^\dagger)=(I\otimes V)(\mathrm c^j\text{-}D)(I\otimes V^\dagger)\), as can be verified on the control subspaces.
Taking \(D=R_0\) shows that every controlled application of \(R_\psi\) uses one ordinary query each to \(\mathcal O\) and \(\mathcal O^\dagger\), whereas controlled applications of \(R_\Pi\) use no oracle queries.
Controlled products can be implemented factorwise, and the controlled global phase in the Grover operator \(-R_\psi R_\Pi\) does not depend on the oracle.
Hence every reflection counted in Lemma~\ref{lem:ndub-amplitude}, controlled or not, uses at most one query each to \(\mathcal O\) and \(\mathcal O^\dagger\), and the stated expected query bound follows after absorbing constants into \(C_Q\).
\end{proof}

\begin{remark}[Controlled oracle access is not needed]
\label{rem:controlled}
Controlled versions of \(\mathcal O\) and \(\mathcal O^\dagger\) are in general not available from black-box access to an unknown unitary, and our model does not grant them.
The proof of Lemma~\ref{lem:lvqme} shows that they are not needed, and the same accounting covers the controlled powers of the Grover operator in the QMC estimator of Lemma~\ref{lem:qmc}, which explains why QMC-based algorithms, including those of \citet{wan2023quantum}, run in this model.
Our lower bounds in Section~\ref{sec:lower-bound} allow tests with controlled queries, and therefore hold a fortiori for policies in our model.
\end{remark}

\section{Lower Bounds}
\label{sec:lower-bound}

We prove the quantum bandit lower bounds in this section. 
We first establish a minimax lower bound for quantum multi-armed bandits and then adapt the same hard instance construction to finite-action quantum linear bandits.

For the lower bounds, it suffices to work with canonical Bernoulli reward oracles. 
For \(p\in[0,1]\), let \(\mathcal O_p\) denote the \emph{canonical Bernoulli oracle}, whose full unitary action is specified in Section~\ref{subsec:single-arm-testing}.
In particular, it acts on the all-zero state as
\[
    \mathcal O_p|0\rangle
    =
    \sqrt{1-p}\,|0\rangle|0\rangle
    +
    \sqrt{p}\,|1\rangle|1\rangle,
\]
and is a specialization of~\eqref{eq:oracle}.
Let \(\mathcal B_K\) denote the class of \(K\)-armed bandit instances with mean function \(\mu:[K]\to[0,1]\) in which every arm uses a canonical Bernoulli oracle, meaning that \(\mathcal O_i=\mathcal O_{\mu(i)}\) for every \(i\in[K]\).
We first prove the following minimax lower bound over \(\mathcal B_K\).

\begin{theorem}[QMAB minimax lower bound]
\label{thm:qmab-lb}
There exists a universal constant \(c>0\) such that, for every \(K\ge2\), \(T\ge1\), and any quantum multi-armed bandit policy \(\pi\) over \(T\) rounds, there exists an instance \(\nu\in\mathcal B_K\) such that
\[
    R_T^\pi(\nu)\ge cK\log(1+\frac{T}{K}).
\]
\end{theorem}

\subsection{Hard Instances and Proof Overview}
\label{subsec:hard-instances-overview}
We construct a family of hard instances inside \(\mathcal B_K\).
The \emph{baseline instance} \(\nu^0\) has mean \(1/2\) on arm \(1\) and mean \(5/12\) on every arm \(i\ge2\).
For each \(i\in\{2,\ldots,K\}\) and \(q\in[7/12,2/3]\), let \(\nu^{i,q}\) be the instance in which arm \(i\) has mean \(q\), arm \(1\) has mean \(1/2\), and every arm \(j\notin\{1,i\}\) has mean \(5/12\).
The set of alternatives for arm \(i\) is \(\mathcal V_i=\{\nu^{i,q}:q\in[7/12,2/3]\}\).
Therefore, arm \(1\) is uniquely optimal in \(\nu^0\), and every suboptimal arm \(i\ge2\) has gap \(1/12\) under \(\nu^0\).
In every instance of \(\mathcal V_i\), arm \(i\) is uniquely optimal, and every suboptimal arm has gap at least \(1/12\).
We denote the hard family by \(\mathcal U_K=\{\nu^0\}\cup\bigcup_{i=2}^K\mathcal V_i\), and for a policy \(\pi\), we write \(M_T=\sup_{\nu\in\mathcal U_K}R_T^\pi(\nu)\) for its worst-case regret over the hard family.

\paragraph{Proof overview.}
The proof turns regret minimization on \(\mathcal U_K\) into a collection of single-arm quantum testing problems. 
Under the baseline instance \(\nu^0\), every round selecting a suboptimal arm \(i\ge2\) incurs constant regret.
By contrast, under the alternative instances in \(\mathcal V_i\), arm \(i\) is optimal and every round selecting another arm incurs constant regret.
A low-regret policy must therefore determine, for each arm \(i\ge2\), whether its mean is the baseline value \(5/12\) or lies in the interval \([7/12,2/3]\). 
It is essential that the alternatives form an interval rather than a single point.

For quantum queries, two fixed unitary oracles can be distinguished without paying a \(\log(1/\delta)\) query cost to reach error probability \(\delta\). Indeed, a constant number of coherent queries suffices to distinguish \(\mathcal O_{5/12}\) and \(\mathcal O_{7/12}\) perfectly~\citep{acin2001statistical,duan2007entanglement}.
As we show below, the logarithmic dependence on the error probability \(\delta\) reappears when the same test must succeed for every mean in the alternative interval \([7/12,2/3]\).
This is why the hard family uses a point-versus-interval testing problem rather than a point-versus-point one.

The main technical step is the following single-arm quantum testing lower bound.
Consider a canonical Bernoulli oracle \(\mathcal O_p\) with unknown mean \(p\).
We prove that any quantum test that queries \(\mathcal O_p\) or its inverse and distinguishes the case \(p=5/12\) from the alternatives \(p\in[7/12,2/3]\), with both error probabilities at most \(\delta\), must make \(\Omega(\log(1/\delta))\) queries in expectation when \(p=5/12\).
We first prove the lower bound for fixed-query tests, which make at most \(q\) queries.
This bound holds under asymmetric error requirements, allowing a constant error at the baseline and requiring error at most \(\delta\) under every alternative.
Parameterizing \(p=\sin^2\theta\), the acceptance probability of such a test is a trigonometric polynomial in \(\theta\) of degree \(O(q)\).
A Remez-type inequality shows that a trigonometric polynomial cannot be bounded away from zero at one angle while being uniformly at most \(\delta\) on an interval of angles unless its degree is \(\Omega(\log(1/\delta))\).
A truncation argument then uses this asymmetry to turn the fixed-query lower bound into the expected-query bound needed for the bandit reduction.

The bandit reduction turns a policy with worst-case regret \(M_T\) into a family of single-arm tests, one for each arm \(i\ge2\).
The test for arm \(i\) decides according to whether arm \(i\) is selected in more than \(T/2\) rounds.
Declaring the wrong side means that at least half of the rounds are spent on suboptimal arms, which incurs \(\Omega(T)\) regret, while the expected regret is at most \(M_T\).
Thus each test has both error probabilities \(O(M_T/T)\).
The testing lower bound, applied with \(\delta=O(M_T/T)\), shows that the test makes \(\Omega(\log(T/M_T))\) queries in expectation under \(\nu^0\).
Since the test queries the unknown oracle only in rounds selecting arm \(i\), and at most once in each such round, this implies \(\mathbb E_{\nu^0}[N_i(T)]=\Omega(\log(T/M_T))\), where \(N_i(T)\) denotes the number of rounds selecting arm \(i\).
Since each such round incurs constant regret under \(\nu^0\), summing over \(i=2,\ldots,K\) gives \(M_T\gtrsim K\log(T/M_T)\).
Solving this implicit inequality and handling small horizons separately yields \(M_T=\Omega(K\log(1+T/K))\).

The next subsections make this argument formal. 
Section~\ref{subsec:single-arm-testing} proves the single-arm quantum testing lower bound.
Section~\ref{subsec:bandit-to-testing} gives the bandit-to-testing reduction.
Section~\ref{subsec:completing-qmab-lb} combines the two to complete the proof of Theorem~\ref{thm:qmab-lb}. 
Finally, Section~\ref{subsec:finite-action-qlb-lower-bound} establishes a lower bound for finite-action quantum linear bandits.

\subsection{A Single-Arm Quantum Testing Lower Bound}
\label{subsec:single-arm-testing}

We now prove the single-arm quantum testing lower bound discussed in the proof overview.

Consider a canonical Bernoulli oracle \(\mathcal O_p\) with unknown mean \(p\in[0,1]\). 
A \emph{single-arm quantum test} starts from a fixed initial state and interleaves queries to \(\mathcal O_p\), \(\mathcal O_p^\dagger\), or their controlled versions with other quantum operations that do not depend on \(p\). 
It may perform intermediate measurements and choose later operations adaptively, but the only \(p\)-dependent operations are the oracle queries.
The test must decide whether the unknown mean is the baseline value or belongs to the alternative interval. 
We write this testing problem as
\[
    H_0:p=\frac{5}{12},
    \qquad
    H_1:p\in\left[\frac{7}{12},\frac{2}{3}\right].
\]
The output \emph{baseline} corresponds to \(H_0\), and the output \emph{alternative} corresponds to \(H_1\).
We also write \(\Pr_p\) and \(\mathbb E_p\) for probabilities and expectations when the oracle is \(\mathcal O_p\).

The proof has two steps.
First, we prove the lower bound for tests with a fixed query budget.
Then a truncation argument converts it to the expected-query lower bound needed for the bandit reduction in Section~\ref{subsec:bandit-to-testing}.
The fixed-query step is a degree argument.
The polynomial representation follows the polynomial method of \citet{beals2001quantum}.
In the Boolean oracle setting, query lower bounds come from approximate-degree lower bounds for Boolean functions.
By contrast, the oracle in our setting is a unitary depending on a continuous parameter, so the degree lower bound comes from the following Remez-type inequality for trigonometric polynomials.

\begin{lemma}[{Remez-type inequality; \citealp[Theorem~1]{ganzburg2012remez}}]
\label{lem:trig-remez}
Let \(J\subseteq\mathbb R\) be an interval of length \(\lambda\in(0,2\pi)\). Every real trigonometric polynomial \(G\) of degree at most \(m\) satisfies
\[
    \sup_{\theta\in\mathbb R}\,\lvert G(\theta)\rvert
    \le
    C_\lambda^{\,m}\,
    \sup_{\theta\in J}\,\lvert G(\theta)\rvert,
    \qquad\text{where}\quad
    C_\lambda=\Big(\tfrac{2}{\sin(\lambda/4)}\Big)^{2}.
\]
\end{lemma}

\begin{proof}
By \citet[Theorem~1]{ganzburg2012remez}, applied with \(m\ge1\), every real trigonometric polynomial \(Q\) of degree at most \(m\) satisfies \(\|Q\|_{C((-\pi,\pi])}\le\tfrac12\big(2/\sin(\lambda/4)\big)^{2m}\|Q\|_{C(E)}\) for every measurable set \(E\subseteq(-\pi,\pi]\) of Lebesgue measure at least \(\lambda\in(0,2\pi]\).
Let \(t_0\) be the center of \(J\) and set \(\widetilde G(\theta):=G(\theta+t_0)\), a real trigonometric polynomial of the same degree.
Since \(\lambda<2\pi\), the interval \(J-t_0=[-\lambda/2,\lambda/2]\) is contained in \((-\pi,\pi]\).
Applying the inequality above to \(\widetilde G\) with \(E=[-\lambda/2,\lambda/2]\), and using the \(2\pi\)-periodicity of \(G\) together with \(\tfrac12\le1\), we obtain \(\sup_{\theta\in\mathbb R}\lvert G(\theta)\rvert=\sup_{\theta\in(-\pi,\pi]}\lvert\widetilde G(\theta)\rvert\le C_\lambda^{\,m}\sup_{\theta\in J}\lvert G(\theta)\rvert\).
The case \(m=0\) holds trivially since \(G\) is then constant.
\end{proof}

Throughout this subsection, we write \(p=\sin^2\theta\) with \(\theta\in[0,\pi/2]\) and abbreviate \(\mathcal O_\theta:=\mathcal O_{\sin^2\theta}\).
We take \(\mathcal O_\theta=R_\theta\oplus I\), where 
\[
    R_\theta=
    \begin{pmatrix}
        \cos\theta & -\sin\theta\\
        \sin\theta & \cos\theta
    \end{pmatrix}
\]
acts on the two-dimensional subspace spanned by \(|0\rangle|0\rangle\) and \(|1\rangle|1\rangle\), and the identity acts on its orthogonal complement. 
Then \(\mathcal O_\theta|0\rangle|0\rangle =\cos\theta\,|0\rangle|0\rangle+\sin\theta\,|1\rangle|1\rangle\),
which matches~\eqref{eq:oracle} with \(p=\sin^2\theta\), and every entry of \(\mathcal O_\theta\) and \(\mathcal O_\theta^\dagger\) is a real trigonometric polynomial in \(\theta\) of degree at most one.

We set
\[
    \theta_0=\arcsin\sqrt{5/12},
    \qquad
    J_1=\Big[\arcsin\sqrt{7/12},\arcsin\sqrt{2/3}\Big],
    \qquad
    \lambda_1=\arcsin\sqrt{2/3}-\arcsin\sqrt{7/12}.
\]
Then \(\theta_0\) is the angle corresponding to \(H_0\), while \(J_1\) is the interval of angles corresponding to \(H_1\).
Both \(\lambda_1>0\) and \(\theta_0\notin J_1\) hold.

With this notation, we now state the fixed-query lower bound under asymmetric error requirements.
We allow a constant error at the baseline so that the bound applies to the truncated test in the proof of Lemma~\ref{lem:expected-query}.

\begin{lemma}[Fixed-query lower bound]
\label{lem:fixed-query}
There exists a universal constant \(c_0>0\) such that the following holds for every \(\delta\in(0,1/6]\). 
Let a single-arm quantum test make at most \(q\) queries and satisfy
\[
    \Pr_{5/12}(\text{output alternative})\le\tfrac13
    \qquad\text{and}\qquad
    \sup_{p\in[7/12,2/3]}\Pr_{p}(\text{output baseline})\le\delta .
\]
Then \(q\ge c_0\log(1/\delta)\).
\end{lemma}

\begin{proof}
We first show that the acceptance probability \(P(\theta)=\Pr_{\sin^2\theta}(\text{output alternative})\) is a real trigonometric polynomial in \(\theta\) of degree at most \(2q\).
By the deferred measurement principle~\citep{nielsen2010quantum}, a test that makes at most \(q\) queries is equivalent to a single quantum circuit in which \(q\) controlled-query steps alternate with \(\theta\)-independent unitaries, followed by a single measurement. 
For \(t=1,\ldots,q\), the \(t\)-th controlled-query step applies
\[
    W_t(\theta)
    =
    \Pi^{(t)}_0\otimes I
    +
    \Pi^{(t)}_+\otimes\mathcal O_\theta
    +
    \Pi^{(t)}_-\otimes\mathcal O_\theta^\dagger,
\]
where \(\Pi^{(t)}_0,\Pi^{(t)}_+,\Pi^{(t)}_-\) are orthogonal projections on the work registers that resolve the identity. 
They correspond respectively to making no query, querying \(\mathcal O_\theta\), and querying \(\mathcal O_\theta^\dagger\) at step \(t\). 
Branches of the computation that have already stopped act through the identity component. 
Therefore, a test that makes at most \(q\) queries can be padded to \(q\) controlled-query steps.
Since every entry of \(\mathcal O_\theta\) and \(\mathcal O_\theta^\dagger\) has degree at most one, so does every entry of \(W_t(\theta)\). 
The final state is
\[
    U_qW_q(\theta)U_{q-1}\cdots W_1(\theta)U_0\lvert\psi_0\rangle,
\]
where \(\lvert\psi_0\rangle\) is the fixed initial state of the test and the \(U_t\) are \(\theta\)-independent unitaries.
Therefore, each amplitude of the final state is a trigonometric polynomial in \(\theta\) of degree at most \(q\) with complex coefficients. 
The probability \(P(\theta)\) is the squared norm of the projection of the final state onto the subspace corresponding to the output alternative. 
Hence \(P(\theta)\) is a sum of squared moduli of such amplitudes, and therefore a real trigonometric polynomial of degree at most \(2q\).

We now convert the two error conditions into a degree lower bound via Lemma~\ref{lem:trig-remez}.
Define \(G(\theta)=1-P(\theta)\), the probability of outputting baseline at angle \(\theta\). 
Then \(G\) is also a real trigonometric polynomial of degree at most \(2q\), and \(0\le G\le1\). 
Since \(\sin^2\theta_0=5/12\), the error condition at \(p=5/12\) gives \(P(\theta_0)\le\tfrac13\), and hence \(G(\theta_0)\ge\tfrac23\).
Since \(J_1\) is the set of angles corresponding to \(p\in[7/12,2/3]\), the error condition on the alternative interval gives \(\sup_{\theta\in J_1}\lvert G(\theta)\rvert\le\delta\), using \(G\ge0\).
Applying Lemma~\ref{lem:trig-remez} with \(J=J_1\), we obtain
\[
    \tfrac23
    \le
    G(\theta_0)
    \le
    \sup_{\theta\in\mathbb R}\lvert G(\theta)\rvert
    \le
    C_{\lambda_1}^{\,2q}\delta .
\]
Therefore \(2q\log C_{\lambda_1}\ge\log(2/(3\delta))\). 
Since \(\delta\le1/6\), we have \(\log\big(2/(3\delta)\big)\ge\tfrac12\log(1/\delta)\). 
Moreover, \(C_{\lambda_1}>1\) is a universal constant determined by the interval
\(J_1\).
Combining these, we obtain \(q\ge c_0\log(1/\delta)\) for a universal constant \(c_0=\big(4\log C_{\lambda_1}\big)^{-1}>0\).
\end{proof}

Lemma~\ref{lem:fixed-query} concerns tests with a fixed query budget, while the bandit reduction in Section~\ref{subsec:bandit-to-testing} produces tests whose number of queries is random, because a bandit policy chooses which arm to query adaptively. 
We next derive the corresponding expected-query lower bound by truncation. Here we only require that the number of queries made so far is known during the run, so that the test can be stopped after a prescribed number of queries. This holds for the tests obtained from bandit policies.

\begin{lemma}[Expected-query lower bound]
\label{lem:expected-query}
There exist universal constants \(c_1,\delta_0>0\) such that the
following holds for every \(\delta\in(0,\delta_0]\). Let a single-arm
quantum test make a random number \(N\ge0\) of queries, with the
number of queries made so far known during the run, and suppose that
\[
    \Pr_{5/12}(\text{output alternative})\le\delta
    \qquad\text{and}\qquad
    \sup_{p\in[7/12,2/3]}\Pr_{p}(\text{output baseline})\le\delta .
\]
Then \(\mathbb E_{5/12}[N]\ge c_1\log(1/\delta)\).
\end{lemma}

\begin{proof}
Let \(\delta_0:=\min\{1/6,\,C_{\lambda_1}^{-8}\}\) and \(c_1:=c_0/12\), where \(c_0\) is the constant of Lemma~\ref{lem:fixed-query}. 
Let \(m:=\mathbb E_{5/12}[N]\). If \(m=\infty\), there is nothing to prove, so assume \(m<\infty\). 
Set the budget \(q:=\max\{1,\lceil 6m\rceil\}\). 
Consider the truncated test that runs the given test, outputs its answer if it stops within \(q\) queries, and otherwise interrupts it after the \(q\)-th query and outputs alternative. 
Since the number of queries made so far is known during the run, the truncation is well defined, and the truncated test makes at most \(q\) queries.

Under \(p=5/12\), Markov's inequality gives \(\Pr_{5/12}(N>q)\le m/q\le 1/6\), so the truncated test outputs alternative with probability at most \(\delta+1/6\le1/3\). 
Under any \(p\in[7/12,2/3]\), forcing the output alternative on truncation cannot increase the probability of outputting baseline, so this probability remains at most \(\delta\). 
The truncated test is thus a single-arm quantum test making at most \(q\) queries and satisfying the two error conditions of Lemma~\ref{lem:fixed-query}.

Applying Lemma~\ref{lem:fixed-query} to the truncated test gives \(q\ge c_0\log(1/\delta)\). Since
\(\delta\le C_{\lambda_1}^{-8}\) and \(c_0=(4\log C_{\lambda_1})^{-1}\), this implies \(q\ge2\). Hence \(q=\lceil6m\rceil\), so \(q\le6m+1\) and \(m\ge(q-1)/6\). Since \(q\ge2\), we also have \(q-1\ge q/2\). Therefore
\[
    m\ge \frac{q-1}{6}\ge \frac{q}{12}
    \ge \frac{c_0}{12}\log(1/\delta)
    = c_1\log(1/\delta).
\]
\end{proof}

\subsection{From Bandit Regret to Testing}
\label{subsec:bandit-to-testing}
We now carry out the reduction from bandit regret to single-arm testing described in the proof overview. Fix a quantum multi-armed bandit policy \(\pi\), and recall that \(M_T=\sup_{\nu\in\mathcal U_K}R_T^\pi(\nu)\) denotes its worst-case regret over the hard family. 
For an instance \(\nu\in\mathcal U_K\) and an arm \(i\in[K]\), let \(N_i(T)\) denote the number of rounds in which \(\pi\) selects arm \(i\) when run on \(\nu\).
Note that \(N_i(T)\) counts selections rather than oracle queries.
The two coincide for policies that query the selected arm's oracle in every round, and in general the number of queries to arm \(i\) is at most \(N_i(T)\).

\begin{lemma}[Bandit-to-testing reduction]
\label{lem:bandit-to-testing}
Let \(c_1,\delta_0\) be the constants of Lemma~\ref{lem:expected-query}, and suppose \(0<M_T\le\delta_0T/24\). 
Then, for every \(i\in\{2,\ldots,K\}\),
\[
    \mathbb E_{\nu^0}\big[N_i(T)\big]
    \ge
    c_1\log\!\Big(\frac{T}{24\,M_T}\Big).
\]
\end{lemma}

\begin{proof}
Fix \(i\in\{2,\ldots,K\}\), and let \(A_i:=\{N_i(T)>T/2\}\).

We first bound the two error probabilities that will arise from the testing rule.

\begin{itemize}
    \item Under \(\nu^0\), arm \(i\) is suboptimal with gap \(1/12\), so each round selecting arm \(i\) incurs regret \(1/12\). Using \(\mathbb E_{\nu^0}[N_i(T)]\ge \frac{T}{2}\Pr_{\nu^0}(N_i(T)>T/2)\), we get
\[
M_T \ge R_T^\pi(\nu^0)
\ge \tfrac{1}{12}\,\mathbb E_{\nu^0}[N_i(T)]
\ge \tfrac{1}{12}\cdot\tfrac{T}{2}\,\Pr_{\nu^0}(A_i)
= \tfrac{T}{24}\,\Pr_{\nu^0}(A_i),
\]
hence \(\Pr_{\nu^0}(A_i)\le 24M_T/T\).
\item Under any instance \(\nu\in\mathcal V_i\), arm \(i\) is optimal and every other arm has gap at least \(1/12\).
Since the policy selects exactly one arm per round, \(\sum_{j\ne i}N_j(T)=T-N_i(T)\), and on the complement \(A_i^c\) we have \(T-N_i(T)\ge T/2\).
Hence
\[
M_T
\ge R_T^\pi(\nu)
\ge \tfrac{1}{12}\,\mathbb E_{\nu}\big[T-N_i(T)\big]
\ge \tfrac{1}{12}\cdot\tfrac{T}{2}\,\Pr_{\nu}(A_i^c)
= \tfrac{T}{24}\,\Pr_{\nu}(A_i^c),
\]
hence \(\Pr_{\nu}(A_i^c)\le 24M_T/T\).
\end{itemize}

We now convert the policy into a single-arm quantum test whose unknown oracle represents arm \(i\). The instance \(\nu^0\) and the instances in \(\mathcal V_i\) differ only in the mean of arm \(i\), so the oracles of all arms \(j\ne i\) are fixed and known. 
The test simulates the policy round by round. Whenever the policy selects an arm \(j\ne i\) and queries its oracle, the test applies the known oracle of arm \(j\) or its inverse, without querying the unknown oracle.
Whenever the policy selects arm \(i\) and queries its oracle, the test queries the unknown oracle or its inverse, matching the operation chosen by the policy.
In rounds where the policy makes no query, the test makes none.
After \(T\) rounds, the test uses \(A_i\) as its decision rule, outputting alternative if \(A_i\) occurs and baseline otherwise.

By construction, running the test with unknown oracle \(\mathcal O_p\) reproduces the run of the policy on the instance in which arm \(i\) has mean \(p\). If \(p=5/12\), this is the run on \(\nu^0\). 
If \(p\in[7/12,2/3]\), it is the run on the corresponding instance in \(\mathcal V_i\). Therefore the error probability under \(H_0\) is \(\Pr_{\nu^0}(A_i)\), and the worst-case error probability under \(H_1\) is \(\sup_{\nu\in\mathcal V_i}\Pr_\nu(A_i^c)\). Both are at most \(24M_T/T\). 
The number of queries made by the test under \(H_0\) is at most the number of rounds in which the simulated policy selects arm \(i\).
The latter has the same distribution as \(N_i(T)\) under \(\nu^0\).
The number of queries made so far can be tracked during the simulation, because each round contains either no oracle call or one explicitly specified call to the selected arm's oracle or its inverse.

Since \(0<24M_T/T\le\delta_0\), we may apply Lemma~\ref{lem:expected-query} with \(\delta=24M_T/T\) to the test above.
Its expected number of queries under \(H_0\) is therefore at least \(c_1\log\frac{T}{24M_T}\), and since this query count is at most the number of rounds selecting arm \(i\), we conclude that
\[
    \mathbb E_{\nu^0}[N_i(T)]
    \ge
    c_1\log\frac{T}{24M_T}.
\]
\end{proof}

\subsection{Completing the QMAB Lower Bound}
\label{subsec:completing-qmab-lb}

We now prove Theorem~\ref{thm:qmab-lb} by applying Lemma~\ref{lem:bandit-to-testing} across arms and evaluating the regret under the baseline instance.

\begin{proof}[Proof of Theorem~\ref{thm:qmab-lb}]
Fix a quantum multi-armed bandit policy \(\pi\), and recall that \(M_T=\sup_{\nu\in\mathcal U_K}R_T^\pi(\nu)\) denotes its worst-case regret over the hard family \(\mathcal U_K\).
Since \(\mathcal U_K\subseteq\mathcal B_K\), it suffices to find an instance in \(\mathcal U_K\) with the claimed regret lower bound.

Let \(c_1,\delta_0\) be the constants of Lemma~\ref{lem:expected-query}, and assume \(T\ge CK\) for a universal constant \(C\) chosen below.
For convenience, set \(x:=T/K\). We prove that \(M_T\ge\bar cK\log x\), where
\(\bar c:=\min\{\delta_0/24,\,c_1/96\}\), distinguishing two cases according to the size of \(M_T\).

\textbf{Case 1:} Suppose \(M_T>\delta_0T/24\), so that the regret is already of
linear order in \(T\). Then
\[
    M_T > \frac{\delta_0}{24}\,Kx
    \ge \frac{\delta_0}{24}\,K\log x
    \ge \bar c\,K\log x ,
\]
using \(\log x\le x\) and \(\bar c\le\delta_0/24\).

\textbf{Case 2:} Suppose \(M_T\le\delta_0T/24\), the regime of Lemma~\ref{lem:bandit-to-testing}.
We first note that \(M_T>0\), which is needed to apply the lemma.
Let \(a_1\) be the first arm selected by the policy.
Consider the alternative instance \(\nu^{2,7/12}\in \mathcal V_2\), in which arm \(2\) has mean \(7/12\).
When \(a_1\) is selected, no query has been made yet, so \(a_1\) has the same distribution under \(\nu^0\) and under \(\nu^{2,7/12}\).
The first round regret is at least \((1/12)\Pr(a_1\ne1)\) under \(\nu^0\), and at least \((1/12)\Pr(a_1\ne2)\) under \(\nu^{2,7/12}\).
Since \(\Pr(a_1\ne1)+\Pr(a_1\ne2)\ge1\), one of these two instances has positive expected regret, and hence \(M_T>0\).

Since \(0<M_T\le\delta_0T/24\), Lemma~\ref{lem:bandit-to-testing}
gives \(\mathbb E_{\nu^0}[N_i(T)]\ge c_1\log\frac{T}{24M_T}\)
for every \(i\in\{2,\ldots,K\}\).
Moreover, under \(\nu^0\), every arm \(i\ge2\) has gap \(1/12\), so each round selecting an arm \(i\ge2\) incurs regret \(1/12\).
Summing over arms and evaluating the regret under \(\nu^0\), we get
\[
M_T \ge R_T^\pi(\nu^0)
= \frac{1}{12}\sum_{i=2}^{K}\mathbb E_{\nu^0}\big[N_i(T)\big]
\ge \frac{c_1(K-1)}{12}\,\log\frac{T}{24M_T}
\ge \frac{c_1}{24}\,K\log\frac{T}{24M_T},
\]
using \(K-1\ge K/2\) for \(K\ge2\).

We now show that \(M_T\ge(c_1/96)K\log x\).
Suppose for contradiction that \(M_T<(c_1/96)K\log x\).
Let \(C\ge e\) be a universal constant such that \(\sqrt{x}\ge(c_1/4)\log x\) for all \(x\ge C\). Then, using \(T=Kx\),
\[
    \frac{T}{24M_T}
    >
    \frac{Kx}{24\cdot(c_1/96)K\log x}
    =
    \frac{4x}{c_1\log x}
    \ge
    \sqrt{x}.
\]
Substituting this bound into the inequality \(M_T\ge(c_1/24)K\log(T/(24M_T))\) yields
\[
    M_T
    \ge
    \frac{c_1}{24}\,K\log\sqrt{x}
    =
    \frac{c_1}{48}\,K\log x
    \ge
    2\cdot\frac{c_1}{96}\,K\log x,
\]
contradicting the assumption.
Hence \(M_T\ge(c_1/96)K\log x\), and since \(\bar c\le c_1/96\), we conclude that \(M_T\ge\bar cK\log x\) in this case as well.

It remains to pass from \(M_T\) to a single instance. 
Since \(M_T\ge\bar cK\log x>0\) and \(M_T\) is the supremum of \(R_T^\pi(\nu)\) over \(\nu\in\mathcal U_K\), there exists an
instance \(\nu\in\mathcal U_K\) with \(R_T^\pi(\nu)\ge M_T/2\ge(\bar c/2)K\log x\).
Since \(x=T/K\ge C\ge e\), we have \(\log(1+x)\le2\log x\).
Hence this instance satisfies \(R_T^\pi(\nu)\ge(\bar c/4)K\log(1+T/K)\).

We now consider \(T<CK\).
Consider the \(K\) instances \(\nu^0,\nu^{2,7/12},\ldots,\nu^{K,7/12}\in\mathcal U_K\).
Their optimal arms are \(1,2,\ldots, K\), respectively, and all suboptimal gaps are at least \(1/12\).
Let \(N_i(t)\) denote the number of rounds in which the policy \(\pi\) selects arm \(i\) among the first \(t\) rounds.
For each \(i\ge2\), the instances \(\nu^0\) and \(\nu^{i,7/12}\) differ only in the oracle of arm \(i\), so \(\Pr_{\nu^{i,7/12}}[N_i(t)=0]=\Pr_{\nu^0}[N_i(t)=0]\).

In each instance, round \(t\) incurs regret at least \(1/12\) if the optimal arm has not been selected by that round.
Moreover, at most \(t\) distinct arms can be selected in \(t\) rounds, so \(\sum_{i=1}^K\Pr_{\nu^0}[N_i(t)=0]\ge K-t\).
Setting \(n:=\min\{T,K\}\) and averaging over the \(K\) instances, we obtain
\[
\begin{aligned}
    \frac1K\left(R_T^\pi(\nu^0)+\sum_{i=2}^K R_T^\pi(\nu^{i,7/12})\right)
    &\ge \frac1{12K}\sum_{t=1}^n\sum_{i=1}^K
    \Pr_{\nu^0}[N_i(t)=0]
    \ge \frac1{12K}\sum_{t=1}^n(K-t)\\
    &=\frac n{12}\left(1-\frac{n+1}{2K}\right) \ge \frac n{12}\left(1-\frac{K+1}{2K}\right)
    \ge\frac n{48},
\end{aligned}
\]
where the last two inequalities use \(n\le K\) and \(K\ge2\), respectively.

Hence there must exist some instance \(\nu\in\mathcal U_K\) satisfying \(R_T^\pi(\nu)\ge\min\{T,K\}/48\).
When \(T\le K\), we have \(\log(1+T/K)\le T/K\), and thus
\[
R_T^\pi(\nu)
\ge\frac{T}{48}
\ge\frac{K\log(1+T/K)}{48}
\ge\frac{K\log(1+T/K)}{48\log(1+C)},
\]
which follows from that \(C\ge e\) and \(\log(1+C)>1\).

When \(K<T<CK\), we have \(\log(1+T/K)\le\log(1+C)\), and hence
\[
    R_T^\pi(\nu)
    \ge\frac{K}{48}
    \ge\frac{K\log(1+T/K)}{48\log(1+C)}.
\]
Combining this with the bound for \(T\ge CK\) proves Theorem~\ref{thm:qmab-lb} with \(c:=\min\{\bar c/4,\,1/(48\log(1+C))\}\).
\end{proof}

\subsection{A Finite-Action QLB Lower Bound}
\label{subsec:finite-action-qlb-lower-bound}

We now prove a lower bound for finite-action quantum linear bandits by realizing the QMAB hard family with an action set of only \(d\) actions.
Embedding the arms directly as the standard basis vectors would require every coordinate of \(\theta\) to carry a constant baseline mean, forcing \(\|\theta\|_2=\Theta(\sqrt d)\).
We instead let the actions \(x_i\), \(i\ge2\), share a common reference direction, so that one coordinate of \(\theta\) provides their common baseline mean while keeping \(\|\theta\|_2=O(1)\).
The orthogonal component of \(x_i\) then encodes the alternatives under which \(x_i\) is optimal.

\begin{theorem}[Finite-action QLB lower bound]
\label{thm:qlb-lb}
There exists a universal constant \(c>0\) such that for every \(d\ge2\) and \(T\ge1\), there exist an action set
\(\mathcal A\subseteq\mathbb R^d\), \(|\mathcal A|=d\), with \(\|x\|_2\le1\) for all \(x\in\mathcal A\), and a parameter set
\(\Theta\subseteq\{\theta\in\mathbb R^d:\|\theta\|_2\le1,\ x^\top\theta\in[0,1]\ \text{for all}\ x\in\mathcal A\}\) such that
for any quantum linear bandit policy \(\pi\) over \(T\) rounds, there exists \(\theta\in\Theta\) with
\[
    R_T^\pi(\theta)\ge cd\log\left(1+\frac{T}{d}\right).
\]
\end{theorem}

\begin{proof}
Let \(c\) be the constant in Theorem~\ref{thm:qmab-lb}. Define the action set \(\mathcal A:=\{x_1,\ldots,x_d\}\) by
\[
    x_1:=e_1,
    \qquad
    x_i:=\tfrac56e_1+\tfrac{\sqrt{11}}6\,e_i,
    \quad i=2,\ldots,d,
\]
so that \(\|x_1\|_2=1\) and \(\|x_i\|_2^2=25/36+11/36=1\).
Define \(\theta^0:=\tfrac12e_1\), and for \(i\in\{2,\ldots,d\}\) and \(q\in[7/12,2/3]\),
\[
    \theta^{(i,q)} :=
    \tfrac12e_1+\tfrac6{\sqrt{11}}\Big(q-\tfrac5{12}\Big)e_i .
\]
Let \(\Theta:=\{\theta^0\}\cup\{\theta^{(i,q)}:2\le i\le d,\ q\in[7/12,2/3]\}\). For each \(\theta\in\Theta\), the corresponding instance uses canonical Bernoulli oracles, with \(\mathcal O_x=\mathcal O_{x^\top\theta}\) for every \(x\in\mathcal A\).

We now verify the parameter and mean constraints. For the parameter norm, \(\|\theta^0\|_2=1/2\), and since \(q-5/12\le1/4\), we have \(\|\theta^{(i,q)}\|_2^2 =1/4+(36/11)(q-5/12)^2 \le 1/4+(36/11)\cdot(1/16) =5/11<1\).
For the means, \(x_1^\top\theta^0=1/2\) and \(x_i^\top\theta^0=(5/6)\cdot(1/2)=5/12\) for \(i\ge2\).
Under \(\theta^{(i,q)}\), the coefficients cancel to give \(x_i^\top\theta^{(i,q)}=5/12+\tfrac{\sqrt{11}}6\cdot\frac{6}{\sqrt{11}}\,(q-5/12)=q\),
while \(x_1^\top\theta^{(i,q)}=1/2\) and \(x_j^\top\theta^{(i,q)}=5/12\) for \(j\notin\{1,i\}\).
In particular, all means lie in \(\{5/12,1/2\}\cup[7/12,2/3]\subseteq[0,1]\).

Consequently, as \(\theta\) ranges over \(\Theta\), the induced action-mean vectors follow exactly the pattern of the hard family \(\mathcal U_d\subseteq\mathcal B_d\) of Section~\ref{subsec:hard-instances-overview}, instantiated with \(K=d\) arms. The parameter \(\theta^0\) induces the baseline instance \(\nu^0\), and \(\theta^{(i,q)}\) induces \(\nu^{i,q}\).
Since a quantum linear bandit policy on the finite action set \(\mathcal A\) selects one action per round and applies its reward oracle or the inverse at most once, it induces a quantum multi-armed bandit policy on the corresponding \(d\)-armed instances, with the same oracle interactions and the same regret. 
Applying the proof of Theorem~\ref{thm:qmab-lb} to this induced \(d\)-armed hard family yields a parameter \(\theta\in\Theta\) with \(R_T^\pi(\theta)\ge cd\log(1+T/d)\).
\end{proof}

\begin{remark}
\label{rem:hard-instances}
The hard instances used in the lower bound proofs have Bernoulli rewards, so both lower bounds also hold under the bounded-variance assumption of \citet{wan2023quantum}.
Moreover, the QLB hard family uses a fixed action set of size \(K=d\), so the QLB lower bound also applies to general action sets.
On the other hand, the reduction uses that arms are selected classically in each round, so the lower bounds do not directly extend to models permitting superposed arm selection, such as the \(m\)-arm oracles of \citet{wang2025best} or the fully coherent access model of \citet{wang2021quantum}.
\end{remark}

\section{Quantum Multi-armed Bandits}
\label{sec:qmab-upper-bound}
We establish an \(O(K\log(1+T/K))\) regret upper bound for quantum multi-armed bandits, matching the lower bound in Theorem~\ref{thm:qmab-lb}.
We present the algorithm in Section~\ref{sec:qmab-algo} and give the regret analysis in Section~\ref{sec:qmab-regret-analysis}.

\subsection{Algorithm}
\label{sec:qmab-algo}
Our algorithm follows the \(\mathrm{QUCB}_1\) framework of \citet{wan2023quantum}, with parameter choices guided by Lemma~\ref{lem:qmc-tail}.
For completeness, we give the pseudocode in Algorithm~\ref{alg:qucb}.   

As shown in Algorithm~\ref{alg:qucb}, \(\mathrm{QUCB}_1\) maintains a mean estimate \(\widehat\mu(i)\), a radius \(r_i\), and a query budget \(N_i\) for each arm \(i\).
It first initializes \(r_i\) and \(N_i\) and runs \(\QMC\) for each arm \(i\) to obtain \(\widehat\mu(i)\) (Lines~\ref{line:qucb-init-budget}--\ref{line:qucb-init-estimate}).
It then proceeds in stages indexed by \(s=1,2,\dots\).
At each stage \(s\), it selects an arm \(i_s\) maximizing \(\widehat\mu(i)+r_i\) (Line~\ref{line:qucb-select}), halves \(r_{i_s}\), and updates the corresponding query budget \(N_{i_s}\) (Line~\ref{line:qucb-update-budget}).
During the next \(N_{i_s}\) rounds, it plays arm \(i_s\) and runs \(\QMC(\mathcal O_{i_s},r_{i_s},\delta)\) with the new radius to update \(\widehat\mu(i_s)\) (Line~\ref{line:qucb-update-estimate}).
The algorithm stops after \(T\) oracle queries.

\begin{algorithm}[H]
\caption{\(\mathrm{QUCB}_1\) (adapted from \citet{wan2023quantum})}
\label{alg:qucb}
\DontPrintSemicolon
\LinesNumbered
\SetKwInput{KwParam}{Parameters}
\KwParam{Failure probability \(\delta\in(0,1/32]\)}
\For{\(i=1\to K\)}{
    \label{line:qucb-init-loop}
    \(r_i\gets1\) and \(N_i\gets\left\lceil\frac{C_2}{r_i}\log\frac1\delta\right\rceil\)
    \label{line:qucb-init-budget}\;
    Play arm \(i\) for \(N_i\) consecutive rounds
    \label{line:qucb-init-play}\;
    Run \(\QMC(\mathcal O_i,r_i,\delta)\) to obtain an estimate \(\widehat\mu(i)\) of \(\mu(i)\)
    \label{line:qucb-init-estimate}\;
}
\label{line:qucb-init-end}
\For{each stage \(s=1,2,\ldots\) (terminate when \(T\) queries have been used in total)}{
    \label{line:qucb-stage-loop}
    Choose \(i_s\in\arg\max_{i\in[K]}\{\widehat\mu(i)+r_i\}\), breaking ties arbitrarily
    \label{line:qucb-select}\;
    Update \(r_{i_s}\gets r_{i_s}/2\) and \(N_{i_s}\gets\left\lceil\frac{C_2}{r_{i_s}}\log\frac1\delta\right\rceil\)
    \label{line:qucb-update-budget}\;
    Play arm \(i_s\) for the next \(N_{i_s}\) rounds, and update \(\widehat\mu(i_s)\) by running \(\QMC(\mathcal O_{i_s},r_{i_s},\delta)\)
    \label{line:qucb-update-estimate}\;
}
\label{line:qucb-stage-end}
\end{algorithm}

\subsection{Regret Analysis}
\label{sec:qmab-regret-analysis}

The following theorem gives the regret guarantee for \(\mathrm{QUCB}_1\).

\begin{theorem}[Regret of \(\mathrm{QUCB}_1\)]
\label{thm:qmab-upper-bound}
For any confidence parameter \(\delta\in(0,1/32]\), with probability at least \(1-K\delta\log_2(1+T/K)\), \(\mathrm{QUCB}_1\) satisfies \( R(T)=O\left(K\log\frac1\delta\right)\).
Moreover, running \(\mathrm{QUCB}_1\) with confidence parameter
\(\delta=\min\{1/32,K/T\}\) gives
\[
    \mathbb E[R(T)]
    =O\left(K\log\left(1+\frac TK\right)\right).
\]
\end{theorem}

\begin{proof}
We first give the main idea of the proof.
When all mean estimates are accurate, a suboptimal arm stops being selected once its radius is sufficiently small relative to its gap.
We use this fact to bound the cumulative regret of each suboptimal arm, and summing these bounds gives the high probability regret bound.
For the expected regret, we first bound the regret incurred during initialization and stages in which the selected arm's radius is large relative to its gap.
We then use Lemma~\ref{lem:qmc-tail} to bound how far the optimal
arm's upper confidence bound falls below its mean and the probability
of a large overestimate of a suboptimal arm.
These bounds control the expected regret in the remaining stages.

For convenience, let \(\widehat\mu_s(i)\) and \(r_{i,s}\) denote the mean estimate and radius of arm \(i\) at the beginning of each stage \(s\).
For each arm \(i\), let \(m_i\) denote the number of stages in which it is selected after initialization.
Define the event \(\mathcal E:=\{|\widehat\mu_s(i)-\mu(i)|\le r_{i,s}\text{ for all }i\in[K]\text{ and all stages }s\}\).

\textbf{High probability regret bound.}
Since \(i_s\) maximizes \(\widehat\mu_s(i)+r_{i,s}\), on the event \(\mathcal E\) we have
\[
\begin{aligned}
    \mu(i^\star)
    \le \widehat\mu_s(i^\star)+r_{i^\star,s} \le \widehat\mu_s(i_s)+r_{i_s,s} \le \mu(i_s)+2r_{i_s,s}.
\end{aligned}
\]
Therefore, \(\Delta_{i_s}\le2r_{i_s,s}\) at every stage \(s\).

Consider a suboptimal arm \(i\). Since the radius starts at \(1\) and is halved each time arm \(i\) is selected, we have
\(\Delta_i\le2^{2-m_i}\).
Including initialization, the algorithm allocates
\(\lceil C_2\log(1/\delta)\rceil\),
\(\lceil 2C_2\log(1/\delta)\rceil,\ldots,
\lceil 2^{m_i}C_2\log(1/\delta)\rceil\)
queries to arm \(i\).
Using \(\lceil x\rceil\le2x\) for \(x\ge1\), their sum is at most \(2C_2(2^{m_i+1}-1)\log(1/\delta)\).
Thus the cumulative regret incurred by arm \(i\) is at most
\[
\begin{aligned}
    2C_2\Delta_i(2^{m_i+1}-1)\log\frac1\delta \le
    2C_2\cdot2^{2-m_i}\cdot2^{m_i+1}\log\frac1\delta =16C_2\log\frac1\delta.
\end{aligned}
\]
Summing over the suboptimal arms gives
\(R(T)\le16C_2(K-1)\log(1/\delta)\) on \(\mathcal E\).

We next bound the probability of \(\mathcal E\).
Following the counting argument of \citet{wan2023quantum},
we have
\[
    \sum_{i=1}^K m_i
    \le K\log_2\left(1+\frac{T}{KC_2\log(1/\delta)}\right)
    \le K\log_2(1+\frac{T}{K}),
\]
where the last inequality follows from \(C_2\log(1/\delta)>1\).

During initialization, \(r_i=1\) and \(\widehat\mu(i),\mu(i)\in[0,1]\), so \(|\widehat\mu(i)-\mu(i)|\le1=r_i\) always holds.
After initialization, the algorithm makes at most
\(\sum_{i=1}^K m_i\) QMC calls.
By Lemma~\ref{lem:qmc-tail} with \(\gamma=\epsilon\), each call satisfies its error bound with failure probability at most \(\delta\).
A union bound therefore gives
\[
    \Pr(\mathcal E)
    \ge1-K\delta\log_2(1+\frac{T}{K}).
\]
Therefore, with probability at least
\(1-K\delta\log_2(1+T/K)\), \(\mathrm{QUCB}_1\) satisfies
\(R(T)=O(K\log(1/\delta))\).

\textbf{Expected regret bound.}
For each suboptimal arm \(i\), we first consider its initialization and stages with \(i_s=i\) and \(r_{i,s}>\Delta_i/4\).
The cumulative regret incurred by arm \(i\) during initialization and these stages is at most
\begin{align}
    \Delta_i\left(
    \left\lceil C_2\log\frac1\delta\right\rceil + \sum_{\substack{s:\,i_s=i\\r_{i,s}>\Delta_i/4}}
    \left\lceil\frac{2C_2}{r_{i,s}}\log\frac1\delta\right\rceil \right) 
    & \le 2C_2\Delta_i\log\frac1\delta
    \left( 1+ \sum_{\substack{s:\,i_s=i\\r_{i,s}>\Delta_i/4}}
    \frac2{r_{i,s}} \right)
    \label{eq:qucb-large-radius-ceiling}\\
    &\le 2C_2\Delta_i\log\frac1\delta \cdot\frac8{\Delta_i}\left(1+\frac12+\frac14+\cdots\right)
    \label{eq:qucb-large-radius-sum}\\
    &= 32C_2\log\frac1\delta,
    \label{eq:qucb-large-radius-regret}
\end{align}
where~\eqref{eq:qucb-large-radius-ceiling} uses
\(\lceil x\rceil\le2x\) for \(x\ge1\), and~\eqref{eq:qucb-large-radius-sum} follows since the radius halves at each update and \(2/r_{i,s}<8/\Delta_i\).
Summing over the suboptimal arms gives \(O(K\log(1/\delta))\).

To bound the regret in the remaining stages, define \(Z\) to measure
how far the optimal arm's upper confidence bound falls below its mean:
\begin{equation}
\label{eq:qmc-z-def}
    Z:=\max_s \big[\mu(i^\star)-(\widehat\mu_s(i^\star)+r_{i^\star,s})\bigr]_+,
\end{equation}
where \((x)_+:=\max\{x,0\}\).
If the algorithm stops before the first stage, we take \(Z=0\).
Thus, at every stage \(s\), \(\widehat\mu_s(i^\star)+r_{i^\star,s}\ge\mu(i^\star)-Z\).

We now consider stages with
\(r_{i_s,s}\le\Delta_{i_s}/4\) and \(\Delta_{i_s}\le2Z\).
Each round in these stages incurs regret at most \(2Z\), so their total regret is at most \(2TZ\).
To control their expected regret, we bound \(\mathbb E[Z]\).
We first consider a single QMC call on arm \(i^\star\) with radius
\(\epsilon\) and output \(\widehat\mu(i^\star)\). 
By Lemma~\ref{lem:qmc-tail}, we have
\begin{align}
    \mathbb E[(\mu(i^\star)-\widehat\mu(i^\star)-\epsilon)_+]
    &=
    \int_0^\infty
    \Pr[\mu(i^\star)-\widehat\mu(i^\star)>\epsilon+v]\,dv
    =\int_\epsilon^\infty \Pr[\mu(i^\star)-\widehat\mu(i^\star)>u]\,du
    \notag\\
    &\le \int_\epsilon^\infty \Pr[|\widehat\mu(i^\star)-\mu(i^\star)|\ge u]\,du
    \le\delta\int_\epsilon^\infty
    \left(\frac{\epsilon}{u}\right)^{\frac12\log(\frac1\delta)}\,du
    \notag\\
    &=\frac{\delta\epsilon}{\frac12\log(1/\delta)-1}
    \le2\delta\epsilon,
    \label{eq:qmc-underestimation}
\end{align}
where the last inequality uses \(\delta\le1/32\).

Recall that the radius \(r_{i^\star,s}\) starts at \(1\) and is halved whenever the estimate of arm \(i^\star\) is updated.
Since the quantities \((\mu(i^\star)-\widehat\mu_s(i^\star)-r_{i^\star,s})_+\) are nonnegative, their maximum is no larger than their sum over the QMC estimates of arm \(i^\star\).
Combining \eqref{eq:qmc-z-def} and \eqref{eq:qmc-underestimation}, we have \(\mathbb E[Z] \le2\delta\left(1+1/2+1/4+\cdots\right) =4\delta\).
Thus the expected regret in these stages is at most
\begin{equation}
\label{eq:qucb-small-gap-regret}
    2T\mathbb E[Z]\le8T\delta.
\end{equation}

We next consider stages with \(r_{i_s,s}\le\Delta_{i_s}/4\) and \(\Delta_{i_s}>2Z\).
Since \(i_s\) maximizes \(\widehat\mu_s(i)+r_{i,s}\), we have \(\widehat\mu_s(i_s)+r_{i_s,s}\ge\mu(i^\star)-Z\), and hence
\[
\widehat\mu_s(i_s)-\mu(i_s)\ge\Delta_{i_s}-Z-r_{i_s,s} >{\Delta_{i_s}}/4.
\]

For each suboptimal arm \(i\), consider a QMC estimate \(\widehat\mu(i)\) with radius \(\epsilon\le\Delta_i/4\).
If arm \(i\) is selected using this estimate in these stages, then
\(\widehat\mu(i)-\mu(i)>\Delta_i/4\).
By Lemma~\ref{lem:qmc-tail},
\begin{equation}
\label{eq:qucb-overestimation-probability}
\begin{aligned}
    \Pr\left[\widehat\mu(i)-\mu(i)>\frac{\Delta_i}4\right]
    \le \delta\left(\frac{4\epsilon}{\Delta_i}\right)^{\frac12\log(1/\delta)} \le
    \delta\left(\frac{4\epsilon}{\Delta_i}\right)^{3/2},
\end{aligned}
\end{equation}
where the last inequality uses \(\delta\le1/32\) and
\(4\epsilon/\Delta_i\le1\).

The stage selecting arm \(i\) with this estimate runs QMC with radius
\(\epsilon/2\) and then updates the estimate.
This stage uses at most \(\lceil(2C_2/\epsilon)\log(1/\delta)\rceil\) queries.
Thus, the regret from using this estimate at such a stage is at most
\[
\Delta_i \left\lceil\frac{2C_2}{\epsilon}\log\frac1\delta\right\rceil
\mathbf1\left\{\widehat\mu(i)-\mu(i)>\frac{\Delta_i}4\right\}.
\]
Taking expectations and applying~\eqref{eq:qucb-overestimation-probability} gives an upper bound of
\[
\begin{aligned}
    \Delta_i \cdot
    \left\lceil\frac{2C_2}{\epsilon}\log\frac1\delta\right\rceil
    \cdot \Pr\left[\widehat\mu(i)-\mu(i)>\frac{\Delta_i}4\right]
     \le \Delta_i\frac{4C_2}{\epsilon}\log\frac1\delta
    \cdot\delta\left(\frac{4\epsilon}{\Delta_i}\right)^{3/2} = 32C_2\delta\sqrt{\frac{\epsilon}{\Delta_i}}\log\frac1\delta.
\end{aligned}
\]

Since \(\epsilon\le\Delta_i/4\), we have \(\sqrt{\epsilon/\Delta_i}\le1/2\).
The radii halve at successive updates of arm \(i\),
so \(\sqrt{\epsilon/\Delta_i}\) is multiplied by \(1/\sqrt2\)
after each update.
Summing over these stages for arm \(i\) therefore gives at most
\begin{equation}
\label{eq:qucb-overestimation-regret}
\begin{aligned}
    32C_2\delta\log\frac1\delta
    \left(\frac12+\frac1{2\sqrt2}+\frac14+\cdots\right)
    = \frac{16C_2\delta}{1-1/\sqrt2}\log\frac1\delta.
\end{aligned}
\end{equation}
Summing over the suboptimal arms gives
\(O(K\delta\log(1/\delta))\).

Summing the bounds in \eqref{eq:qucb-large-radius-regret}
and \eqref{eq:qucb-overestimation-regret} over the suboptimal arms
and adding \eqref{eq:qucb-small-gap-regret} gives
\[
    \mathbb E[R(T)]
    =O\left(K\log\frac1\delta
        +K\delta\log\frac1\delta+T\delta\right)
    =O\left(K\log\frac1\delta+T\delta\right).
\]

Take \(\delta=\min\{1/32,K/T\}\).
When \(T\ge K\), we have \(T\delta\le K\) and \(\log(1/\delta) =\log\max\{32,T/K\} =O\left(\log(1+T/K)\right)\).
Since \(\log(1+T/K)\ge\log2\), it holds that \(K\le K\log(1+T/K)/\log2\).
The expected regret bound gives \(\mathbb E[R(T)] =O\left(K\log(1+T/K)\right)\).
When \(T<K\), each round incurs regret at most \(1\), so \(\mathbb E[R(T)] \le T \le2K\log(1+T/K)\),
where the last inequality uses
\(\log(1+x)\ge x/2\) for \(x\in[0,1]\).
This concludes our proof.
\end{proof}

\section{Finite-Action Quantum Linear Bandits}
\label{sec:finite-action-qlb}

Section~\ref{subsec:finite-action-qlb-lower-bound} shows that finite-action quantum linear bandits suffer regret \(\Omega(d\log(1+T/d))\).
Complementing this lower bound, we give \LVGElim, a phased elimination algorithm for finite action sets. 
When \(K=\operatorname{poly}(d)\), \LVGElim has regret nearly linear in \(d\), matching the lower bound up to polylogarithmic factors. 
This improves the \(d^2\) dependence of the QLinUCB algorithm of \citet{wan2023quantum}.

The algorithm uses a small-support approximate \(G\)-optimal design on the current active set in each phase. 
The \(G\)-optimality condition controls the error amplification when the support estimates are extrapolated to all active actions, while the small support keeps the number of queried actions proportional to the dimension.
We estimate the support means using the low-bias low-variance quantum mean estimator from Lemma~\ref{lem:lvqme}, with a query allocation matched to the design
weights, so that the reconstruction error aggregates through variance. 
A direct high-probability estimator such as QMC would instead aggregate worst-case absolute errors and lose an extra \(\sqrt d\) factor. 
We return to this comparison at the end of the section.

\subsection{Small-Support Approximate \ensuremath{G}-Optimal Design}
\label{subsec:g-optimal-design}

Classical phased elimination algorithms for finite-action linear bandits use \(G\)-optimal designs to estimate all active actions from samples drawn on a design support~\citep{lattimore2020bandit}.
In the classical analysis, estimating the active actions to accuracy \(\epsilon\) in a phase costs \(\epsilon^{-2}\) samples in the leading term, and the size of the design support only contributes a lower-order rounding overhead.
Quantum queries change this accounting.
The leading cost of quantum mean estimation is of order \(\epsilon^{-1}\) rather than \(\epsilon^{-2}\), but the estimator must be run separately on each support action, with a constant minimum number of queries per call.
The size of the support thus enters the query bound directly, rather than only as a rounding overhead.
To obtain regret nearly linear in \(d\), the design must therefore have nearly linear support while still controlling the error amplification of extrapolating from the support to all active actions.
The following lemma provides such a design with support size \(O(r)\) and \(G\)-value \(O(r)\), where \(r\) is the dimension of the span of the active set.

\begin{lemma}[Small-support approximate \(G\)-optimal design]
\label{lem:small-support-g-optimal-design}
There exist universal constants \(c_s,C_G>0\) such that the following holds.
Let \(\mathcal X\subseteq\mathbb R^d\) be finite with \(r=\dim(\operatorname{span}(\mathcal X))\ge1\).
One can compute a subset \(S\subseteq\mathcal X\) and a distribution \(\pi\) supported on \(S\), with \(|S|\le c_s r\), such that the design information matrix \(M=\sum_{a\in S}\pi(a)aa^\top\) satisfies
\[
    \operatorname{Range}(M)=\operatorname{span}(\mathcal X),
    \qquad
    \max_{x\in\mathcal X}x^\top M^+x\le C_G r.
\]
For each \(x\in\mathcal X\), define its design-induced prediction weights by \(\alpha_a(x)=\pi(a)a^\top M^+x\) for \(a\in S\). 
Then we have
\[
    x=\sum_{a\in S}\alpha_a(x)a,
    \qquad
    \sum_{a\in S}\frac{\alpha_a(x)^2}{\pi(a)}
    =
    x^\top M^+x
    \le C_G r.
\]
\end{lemma}

\begin{proof}[Proof sketch]
The proof has two steps: we first compute an approximately optimal continuous design, and then round it to an integer design of total mass \(m=\Theta(r)\), whose normalization is the desired distribution \(\pi\). 
Since rounding preserves the \(G\)-value up to a constant factor, and the normalization multiplies back at most the factor \(m\) removed by the scaling, the final \(G\)-value remains \(O(r)\).

For the first step, let \(V=\operatorname{span}(\mathcal X)\) and work in coordinates on \(V\), so that the design problem is \(r\)-dimensional and full rank. In this space, the Kiefer--Wolfowitz equivalence theorem \citep{kiefer1960equivalence} shows that the continuous \(G\)-optimal value is \(r\), and the continuous relaxation method of \citet[Section~3]{allen2021near} computes a distribution \(w\) whose \(G\)-value is \(O(r)\).

For the second step, we round the scaled design \(mw\), whose information matrix is \(m\) times that of \(w\) and whose \(G\)-value is therefore \(O(r/m)\).
To match the constraints of the rounding theorem, we replace each design point by \(m\) identical copies and assign weight \(w(x)\) to each copy. This copied-list representation has the same total mass \(m\) and the same information matrix \(mM_w\) as \(mw\), but all coordinate weights lie in \([0,1]\). 
Applying the rounding theorem of \citet[Theorem~2.1]{allen2021near} with budget \(m\) then yields integer multiplicities of total mass \(\widehat m\le m\) whose \(G\)-value is at most \(O(r/m)\). Merging identical copies and normalizing by \(\widehat m\) gives a distribution \(\pi\) supported on a set \(S\) with \(|S|\le\widehat m=O(r)\), and since \(M\) is the rounded information matrix divided by \(\widehat m\), the \(G\)-value is multiplied by \(\widehat m\le m\).
Thus the resulting matrix \(M\) satisfies \(\max_{x\in\mathcal X}x^\top M^+x \le\widehat m\cdot O(r/m)=O(r)\).

Finally, the rounded design has finite \(G\)-value, so the selected points span \(V\) and hence \(\operatorname{Range}(M)=\operatorname{span}(\mathcal X)\). The prediction-weight identities then follow from \(MM^+x=x\), valid for \(x\in\mathcal X\subseteq\operatorname{Range}(M)\), together with \(M^+MM^+=M^+\). The full proof is given in Appendix~\ref{app:g-optimal-design}.
\end{proof}

The design can be computed in polynomial time.
The running-time accounting is given in Appendix~\ref{app:g-optimal-design}.
We refer to a triple \((S,\pi,M)\) satisfying Lemma~\ref{lem:small-support-g-optimal-design} as a \emph{design}.
The associated prediction weights express the mean of every active action as a weighted combination of the means on the support, with the resulting error amplification controlled by the \(G\)-value.
We next show how to estimate the support means with a query allocation matched to the design weights.

\subsection{From LVQME to High-Probability Design Estimates}
\label{subsec:lvqme-to-hp-design-estimates}

The estimator \(\LVQME\) controls the bias and the variance of its output, rather than directly providing a high-probability error bound.
We therefore proceed in two steps: we first obtain estimates that are accurate with constant probability for each fixed action, and then amplify them to estimates that hold simultaneously over the active set with high probability, as required by the elimination step.

\paragraph{The estimation procedure.}
Fix an active set \(\mathcal A\subseteq\mathbb R^d\) and a design \((S,\pi,M)\) satisfying Lemma~\ref{lem:small-support-g-optimal-design}, and let \(r=\dim(\operatorname{span}(\mathcal A))\ge1\).
Given an accuracy parameter \(\epsilon\in(0,1]\), we set
\[
    t_a:=\max\left\{4,\left\lceil
    c_t\frac{\sqrt{r\pi(a)}}{\epsilon}\right\rceil\right\}
    \quad(a\in S),
    \qquad
    \zeta:=c_\zeta\frac{\epsilon^2}{r},
\]
where \(c_t>0\) is a sufficiently large universal constant and \(c_\zeta>0\) is a sufficiently small universal constant, which are to be chosen in the proof of Lemma~\ref{lem:lvqme-design-estimates}.
The value \(t_a\) allocates more estimation effort to support actions with larger design weight \(\pi(a)\).
Let
\begin{equation}
\label{eq:lvqme-threshold}
    U:=\left\lceil
    6\,C_Q\sum_{a\in S}
    t_a\log\log(t_a)\log\frac{t_a}{\zeta}
    \right\rceil,
\end{equation}
where \(C_Q\) is the query constant in Lemma~\ref{lem:lvqme}.
The procedure runs one independent call \(\LVQME(\mathcal O_a,t_a,\zeta)\) for each \(a\in S\), sequentially in an arbitrary fixed order, using \(U\) as a shared deterministic threshold on the total number of oracle queries. 
If some call is still running after \(U\) oracle queries in total, the procedure stops and sets \(\widehat\mu_a:=0\) for every \(a\in S\).
This stopping event is accounted for in Lemma~\ref{lem:lvqme-design-estimates}.
Otherwise, \(\widehat\mu_a\) is the output of the call on \(a\).
In both cases, the procedure returns
\[
    \widetilde\mu(x):=\sum_{a\in S}\alpha_a(x)\,\widehat\mu_a
    \qquad\text{for all }x\in\mathcal A,
\]
where \(\{\alpha_a(x)\}_{a\in S}\) are the prediction weights from Lemma~\ref{lem:small-support-g-optimal-design}.

The following lemma shows that the estimation procedure is accurate with constant probability for each fixed action \(x\in\mathcal A\).

\begin{lemma}[LVQME design estimates]
\label{lem:lvqme-design-estimates}
There exist universal constants \(c_t,c_\zeta>0\) such that the following holds. 
The estimation procedure has the deterministic query threshold \(U\) as in~\eqref{eq:lvqme-threshold}. 
This threshold satisfies 
\[
    U=O\left(
    \frac{r}{\epsilon}\,
    \log\Big(4+\frac{r}{\epsilon}\Big)
    \log\log\Big(4+\frac{r}{\epsilon}\Big)
    \right).
\]
Moreover, for each fixed \(x\in\mathcal A\),
\[
    \Pr\!\left[\,\lvert\widetilde\mu(x)-\mu(x)\rvert>\epsilon\,\right]
    \le\frac{5}{12}.
\]
\end{lemma}

\begin{proof}
We first give the main idea of the proof. 
The stopping rule converts the random query cost of the \(\LVQME\) calls into the deterministic threshold \(U\), at the price of a constant failure probability, which is controlled by Markov's inequality. 
On the remaining event, the estimate \(\widetilde\mu(x)\) inherits the two guarantees of Lemma~\ref{lem:lvqme}: its bias is small because each \(\widehat\mu_a\) is nearly unbiased and we choose \(\zeta=c_\zeta\epsilon^2/r\), while its variance is small because the fluctuations of independent calls average under the design weights. 
The allocation \(t_a\propto\sqrt{r\pi(a)}/\epsilon\) makes the variance controlled by the design quantity \(x^\top M^+x\).
Chebyshev's inequality then converts this variance bound into the constant probability guarantee. 
We now turn to the details.

By construction, the estimation procedure never makes more than \(U\) oracle queries, so for the query bound it suffices to bound the size of \(U\). First, by the definition of \(t_a\), each \(a\in S\) satisfies \(t_a\le c_t\sqrt{r\pi(a)}/\epsilon+5\), where the ceiling contributes at most \(1\) and the truncation at \(4\) contributes at most \(4\). 
Summing over \(a\in S\) and applying the Cauchy--Schwarz inequality to the weights \(\sqrt{\pi(a)}\),
\[
    \sum_{a\in S}t_a
    \le 5\lvert S\rvert
    +\frac{c_t\sqrt r}{\epsilon}\sum_{a\in S}\sqrt{\pi(a)}
    \le 5\lvert S\rvert
    +\frac{c_t\sqrt r}{\epsilon}
    \sqrt{\lvert S\rvert\sum_{a\in S}\pi(a)}
    \le 5c_sr+\frac{c_t\sqrt{c_s}\,r}{\epsilon}
    = O\Big(\frac{r}{\epsilon}\Big),
\]
using \(\lvert S\rvert\le c_sr\), \(\sum_{a\in S}\pi(a)=1\), and \(\epsilon\le1\) to absorb the first term.
Next, we bound the logarithmic factors. 
Every \(a\in S\) satisfies \(4\le t_a=O(\sqrt r/\epsilon)\), and hence \(t_a/\zeta=t_a\cdot r/(c_\zeta\epsilon^2)
=O(r^{3/2}/\epsilon^{3})\), where the constants depend only on \(c_t,c_\zeta\). Therefore \(\log\log(t_a)\log(t_a/\zeta) =O\big(\log\big(4+\tfrac r\epsilon\big) \log\log\big(4+\tfrac r\epsilon\big)\big)\) for every \(a\in S\), and combining this with the bound on \(\sum_{a\in S}t_a\) gives the stated bound on \(U\).

We now fix \(x\in\mathcal A\) and bound the error probability. 
Consider the process that runs all \(\lvert S\rvert\) calls to completion, with the stopping rule removed, on the same realization of randomness.
The estimation procedure coincides with this uninterrupted process unless the latter makes more than \(U\) queries.
By Lemma~\ref{lem:lvqme}, the expected total number of queries of the uninterrupted process is at most \(C_Q\sum_{a\in S}t_a\log\log(t_a)\log(t_a/\zeta)\le U/6\) by~\eqref{eq:lvqme-threshold}.
Hence, by Markov's inequality, the uninterrupted process makes more than \(U\) queries with probability at most \(1/6\).
Otherwise, the stopping rule is not triggered and the two processes coincide.

It remains to bound the deviation of the uninterrupted estimator, which we again denote by \(\widetilde\mu(x)\). 
For each \(a\in S\), decompose
\[
    \widehat\mu_a-\mu(a)=b_a+\xi_a,
    \qquad
    b_a:=\mathbb E[\widehat\mu_a]-\mu(a),
    \qquad
    \xi_a:=\widehat\mu_a-\mathbb E[\widehat\mu_a],
\]
so that \(\mathbb E[\xi_a]=0\) and Lemma~\ref{lem:lvqme} gives \(\lvert b_a\rvert\le\zeta\) and \(\operatorname{Var}(\xi_a)\le C_V/t_a^2+\zeta\). 
Moreover, the \(\xi_a\) are independent across \(a\in S\), since the procedure runs the calls independently.
Since Lemma~\ref{lem:small-support-g-optimal-design} gives \(x=\sum_{a\in S}\alpha_a(x)a\), the linear model implies \(\mu(x)=\sum_{a\in S}\alpha_a(x)\mu(a)\). Hence
\[
    \widetilde\mu(x)-\mu(x)
    = \sum_{a\in S}\alpha_a(x)\,b_a
    + \sum_{a\in S}\alpha_a(x)\,\xi_a .
\]

For the bias term, the Cauchy--Schwarz inequality weighted by \(\pi\) gives
\[
    \Big\lvert\sum_{a\in S}\alpha_a(x)\,b_a\Big\rvert
    \le
    \zeta\sum_{a\in S}\lvert\alpha_a(x)\rvert
    \le
    \zeta
    \Big(\sum_{a\in S}\frac{\alpha_a(x)^2}{\pi(a)}\Big)^{1/2}
    \Big(\sum_{a\in S}\pi(a)\Big)^{1/2}
    \le
    \zeta\sqrt{C_Gr}
    =
    c_\zeta\sqrt{C_G}\,\frac{\epsilon^2}{\sqrt r}
    \le
    \frac{\epsilon}{4},
\]
using \(\sum_{a\in S}\alpha_a(x)^2/\pi(a)=x^\top M^+x\le C_Gr\), \(r\ge1\), and \(\epsilon\le1\), provided \(c_\zeta\le\big(4\sqrt{C_G}\big)^{-1}\).

For the fluctuation term, since the \(\xi_a\) are independent,
\[
    \operatorname{Var}\Big(\sum_{a\in S}\alpha_a(x)\,\xi_a\Big)
    =
    \sum_{a\in S}\alpha_a(x)^2\operatorname{Var}(\xi_a)
    \le
    C_V\sum_{a\in S}\frac{\alpha_a(x)^2}{t_a^2}
    +
    \zeta\sum_{a\in S}\alpha_a(x)^2 ,
\]
where we substituted \(\operatorname{Var}(\xi_a)\le C_V/t_a^2+\zeta\). 
For the first term, the definition of \(t_a\) gives \(t_a^2\ge c_t^2\,r\,\pi(a)/\epsilon^2\), and hence
\[
    C_V\sum_{a\in S}\frac{\alpha_a(x)^2}{t_a^2}
    \le
    \frac{C_V\epsilon^2}{c_t^2\,r}
    \sum_{a\in S}\frac{\alpha_a(x)^2}{\pi(a)}
    \le
    \frac{C_V\epsilon^2}{c_t^2\,r}\cdot C_Gr
    =
    \frac{C_VC_G}{c_t^2}\,\epsilon^2
    \le
    \frac{\epsilon^2}{32},
\]
using \(\sum_{a\in S}\alpha_a(x)^2/\pi(a)=x^\top M^+x\le C_Gr\), provided \(c_t^2\ge32C_VC_G\). For the second term, since \(\pi(a)\le1\) for every \(a\in S\),
\[
    \zeta\sum_{a\in S}\alpha_a(x)^2
    \le
    \zeta\sum_{a\in S}\frac{\alpha_a(x)^2}{\pi(a)}
    \le
    \zeta\,C_Gr
    =
    c_\zeta C_G\,\epsilon^2
    \le
    \frac{\epsilon^2}{32},
\]
provided \(c_\zeta\le(32C_G)^{-1}\). Combining the two terms,
\[
    \operatorname{Var}\Big(\sum_{a\in S}\alpha_a(x)\,\xi_a\Big)
    \le
    \frac{\epsilon^2}{16}.
\]
Since \(\mathbb E\big[\sum_{a\in S}\alpha_a(x)\,\xi_a\big]=0\), by Chebyshev's inequality,
\[
    \Pr\Big[\Big\lvert\sum_{a\in S}\alpha_a(x)\,\xi_a\Big\rvert
    >\frac{\epsilon}{2}\Big]
    \le\frac14 .
\]

Suppose that the stopping rule is not triggered and that \(\big\lvert\sum_{a\in S}\alpha_a(x)\,\xi_a\big\rvert \le\epsilon/2\). 
In this case, \(\widetilde\mu(x)\) coincides with the uninterrupted estimator, and by the decomposition above,
\[
    \lvert\widetilde\mu(x)-\mu(x)\rvert
    \le \Big\lvert\sum_{a\in S}\alpha_a(x)\,b_a\Big\rvert + \Big\lvert\sum_{a\in S}\alpha_a(x)\,\xi_a\Big\rvert
    \le \frac{\epsilon}{4}+\frac{\epsilon}{2}
    \le \epsilon ,
\]
where the bias bound holds deterministically. Hence the error event \(\{\lvert\widetilde\mu(x)-\mu(x)\rvert>\epsilon\}\) is contained in the union of the two events above, and a union bound gives
\(\Pr\big[\lvert\widetilde\mu(x)-\mu(x)\rvert>\epsilon\big]\le1/6+1/4=5/12\).

Finally, choosing \(c_t:=\sqrt{32C_VC_G}\) and \(c_\zeta:=\min\big\{1,\big(4\sqrt{C_G}\big)^{-1}, \big(32C_G\big)^{-1}\big\}\) satisfies all the constraints above, and \(c_\zeta\le1\) ensures \(\zeta\in(0,1)\) as required by Lemma~\ref{lem:lvqme}.
\end{proof}

\paragraph{High-probability estimates.}
The estimation procedure is accurate for each fixed action only with constant probability. To obtain estimates that are accurate for all actions in \(\mathcal A\) simultaneously with high probability, we run it several times independently and take medians.

\begin{lemma}[High-probability LVQME design estimates]
\label{lem:hp-design-estimates}
There exists a universal constant \(C_N>0\) such that the following holds. 
For a confidence parameter \(\delta\in(0,1)\), let \(N:=\big\lceil C_N\log\frac{\lvert\mathcal A\rvert}{\delta}\big\rceil\).
Run the estimation procedure of Lemma~\ref{lem:lvqme-design-estimates} independently \(N\) times, and for each \(x\in\mathcal A\), let \(\widehat\mu(x)\) be the median of \(\widetilde\mu^{(1)}(x),\ldots,\widetilde\mu^{(N)}(x)\), where \(\widetilde\mu^{(j)}(x)\) denotes the estimate \(\widetilde\mu(x)\) returned by the \(j\)-th run.
Then the total number of oracle queries is deterministically at most
\[
    NU =
    O\left(
    \frac{r}{\epsilon}\,
    \log\frac{\lvert\mathcal A\rvert}{\delta}\,
    \log\Big(4+\frac{r}{\epsilon}\Big)
    \log\log\Big(4+\frac{r}{\epsilon}\Big)
    \right),
\]
and with probability at least \(1-\delta\),
\[
    \lvert\widehat\mu(x)-\mu(x)\rvert\le\epsilon
    \qquad\text{for all }x\in\mathcal A .
\]
\end{lemma}

\begin{proof}
Fix \(x\in\mathcal A\). By Lemma~\ref{lem:lvqme-design-estimates}, each run satisfies \(\lvert\widetilde\mu^{(j)}(x)-\mu(x)\rvert\le\epsilon\) with probability at least \(1-5/12=7/12>1/2\), independently across \(j\in\{1,\ldots,N\}\). 
By a standard median argument and a Chernoff bound,
\[
    \Pr\big[\,\lvert\widehat\mu(x)-\mu(x)\rvert>\epsilon\,\big]
    \le
    e^{-\Omega(N)}
    \le
    \frac{\delta}{\lvert\mathcal A\rvert},
\]
provided \(C_N\) is sufficiently large. 
A union bound over \(x\in\mathcal A\) then shows that, with probability at least \(1-\delta\), \(\lvert\widehat\mu(x)-\mu(x)\rvert\le\epsilon\) for all \(x\in\mathcal A\).

For the query bound, the medians \(\widehat\mu(x)\) for \(x\in\mathcal A\), are computed from the same \(N\) runs, and each run makes at most \(U\) queries by Lemma~\ref{lem:lvqme-design-estimates}. Hence the total number of queries is deterministically at most \(NU\), and the stated bound follows from the bound on \(U\) in Lemma~\ref{lem:lvqme-design-estimates} together with the definition of \(N\).
\end{proof}

For convenience, we refer to this median-amplified estimator as \LVQMEDesignEstimate.

\subsection{Low-Variance \ensuremath{G}-Optimal Design Elimination}
\label{subsec:lv-g-elim}

We now combine the design and the estimator into a phased elimination algorithm.
\LVGElim{} (Algorithm~\ref{alg:lv-g-elim}) follows the classical \(G\)-optimal exploration scheme~\citep{lattimore2020bandit}, with the sampling-and-estimation step replaced by \(\LVQMEDesignEstimate\).
In phase \(\ell\), it computes a small-support design for the current active set \(\mathcal A_\ell\) (Line~\ref{line:g-optimal-design}), calls the estimator with target accuracy \(\epsilon_\ell=2^{-\ell}\) (Line~\ref{line:LVQMEDesignEstimate}), and eliminates every action whose estimate falls more than \(2\epsilon_\ell\) below the best estimate (Lines~\ref{line:best-estimate} and~\ref{line:eliminate}).
The algorithm stops once the total number of oracle queries reaches \(T\). 
If every phase finishes before the budget is exhausted, the algorithm queries an arbitrary action in the current active set for the remaining rounds. 
If \(\operatorname{span}(\mathcal A_\ell)=\{0\}\) at the beginning of a phase, the algorithm stops the phased procedure and plays an arbitrary active action for all remaining rounds.

\begin{algorithm}[H]
\caption{Low-Variance G-Optimal Design Elimination (\LVGElim)}
\label{alg:lv-g-elim}
\DontPrintSemicolon
\KwIn{Action set \(\mathcal A\subseteq\mathbb R^d\), horizon \(T\), confidence parameter \(\delta\in(0,1)\)}
\(L_T\gets\lceil\log_2 T\rceil\), \(\mathcal A_0\gets\mathcal A\), \(\delta^\prime\gets\delta/(L_T+1)\)\;
\For{\(\ell=0,1,\ldots,L_T\) \textnormal{(terminate once the total number of oracle queries reaches \(T\))}}{
    \(\epsilon_\ell\gets 2^{-\ell}\)\;
    Compute a small-support design \((S_\ell,\pi_\ell,M_\ell)\) for \(\mathcal A_\ell\) as in Lemma~\ref{lem:small-support-g-optimal-design} \label{line:g-optimal-design}\;
    \(\{\widehat\mu_\ell(x)\}_{x\in\mathcal A_\ell}\gets\LVQMEDesignEstimate(\mathcal A_\ell,S_\ell,\pi_\ell,M_\ell,\epsilon_\ell,\delta^\prime)\) \label{line:LVQMEDesignEstimate}\;
    \(\widehat\mu_\ell^\star\gets\max_{z\in\mathcal A_\ell}\widehat\mu_\ell(z)\)\label{line:best-estimate}\;
    \(\mathcal A_{\ell+1}\gets\{x\in\mathcal A_\ell:\widehat\mu_\ell(x)\ge \widehat\mu_\ell^\star-2\epsilon_\ell\}\) \label{line:eliminate}\;
}
\end{algorithm}

The following theorem gives the regret guarantee for \LVGElim.

\begin{theorem}[Regret of \LVGElim]
\label{thm:lv-g-elim-regret}
Assume \(T\ge2\) and \(dT\ge3\).
For any confidence parameter \(\delta\in(0,1)\), with probability at least \(1-\delta\), \LVGElim satisfies
\[ R(T) = O\left( d\log T \log\frac{K\log T}{\delta} \log(dT)\log\log(dT) \right).\]
Moreover, running \LVGElim with confidence parameter \(1/T\) gives
\[
\mathbb E[R(T)] = O\left( d\log T \log(KT) \log(dT)\log\log(dT)\right).
\]
\end{theorem}

\begin{proof}
We first give the main idea of the proof. 
With probability at least \(1-\delta\), all estimates used by the algorithm are accurate for the corresponding active actions.
On this good event, the optimal action is never eliminated, and every action that remains active after phase \(\ell\) has suboptimality gap \(O(\epsilon_\ell)\). 
The number of queries in phase \(\ell\) is \(\widetilde O(d/\epsilon_\ell)\) by Lemma~\ref{lem:hp-design-estimates}, where \(\widetilde O\) hides logarithmic factors in \(d,T,K\) and \(1/\delta\), so the regret incurred in phase \(\ell\) is \(\widetilde O(d)\). 
Summing over the \(O(\log T)\) phases gives the stated bound. 
We now turn to the details.

The algorithm runs at most \(L_T+1\) phases. 
For each phase \(\ell=0,\ldots,L_T\) in which the elimination step is performed, define \(\mathcal E_\ell := \big\{\,\lvert\widehat\mu_\ell(x)-\mu(x)\rvert\le\epsilon_\ell \ \text{ for all }x\in\mathcal A_\ell\,\big\}\) and \(\mathcal E:=\bigcap_{\ell}\mathcal E_\ell\),  where the intersection ranges over all such phases. 
For each such phase, let \(\mathcal F_\ell\) be the history before phase \(\ell\). Conditional on \(\mathcal F_\ell\), the active set \(\mathcal A_\ell\) and the design \((S_\ell,\pi_\ell,M_\ell)\) are fixed, so Lemma~\ref{lem:hp-design-estimates} gives \(\Pr(\mathcal E_\ell^c\mid\mathcal F_\ell)\le\delta'\) almost surely.
Taking expectations gives \(\Pr(\mathcal E_\ell^c)\le\delta'\), where \(\delta'=\delta/(L_T+1)\). A union bound over the at most \(L_T+1\) phases gives \(\Pr[\mathcal E]\ge1-\delta\). It remains to prove the regret bound on \(\mathcal E\).

Fix \(x^\star\in\arg\max_{x\in\mathcal A}\mu(x)\), set \(\mu^\star:=\mu(x^\star)\), and write \(\Delta(x):=\mu^\star-\mu(x)\). 
We claim the following on \(\mathcal E\): for every phase \(\ell\) in which the elimination step is performed, \(x^\star\in\mathcal A_{\ell+1}\), and \(\Delta(x)\le4\epsilon_\ell\) for all \(x\in\mathcal A_{\ell+1}\).

Initially, \(x^\star\in\mathcal A_0=\mathcal A\).
Assume inductively that \(x^\star\in\mathcal A_\ell\).
For every \(z\in\mathcal A_\ell\), on \(\mathcal E_\ell\),
\[
    \widehat\mu_\ell(x^\star)
    \ge\mu^\star-\epsilon_\ell
    \ge\mu(z)-\epsilon_\ell
    \ge\widehat\mu_\ell(z)-2\epsilon_\ell .
\]
Thus \(\widehat\mu_\ell(x^\star)\ge\widehat\mu_\ell^\star-2\epsilon_\ell\), so \(x^\star\) is not eliminated. 
If \(x\in\mathcal A_{\ell+1}\), then \(\widehat\mu_\ell(x)\ge\widehat\mu_\ell^\star-2\epsilon_\ell \ge\widehat\mu_\ell(x^\star)-2\epsilon_\ell\), and hence
\[
    \mu(x)
    \ge\widehat\mu_\ell(x)-\epsilon_\ell
    \ge\widehat\mu_\ell(x^\star)-3\epsilon_\ell
    \ge\mu^\star-4\epsilon_\ell .
\]
This proves the claim.

We now bound the instantaneous regret of the actions queried in each phase. 
In phase \(\ell\), the algorithm queries only actions in \(S_\ell\subseteq\mathcal A_\ell\). 
For \(\ell=0\), all gaps are at most \(1=\epsilon_0\), since all means lie in \([0,1]\). 
For \(\ell\ge1\), the claim applied to phase \(\ell-1\) gives \(\Delta(x)\le4\epsilon_{\ell-1}=8\epsilon_\ell\) for all \(x\in\mathcal A_\ell\). 
Thus every query performed in phase \(\ell\) incurs regret at most \(8\epsilon_\ell\), for every \(\ell\ge0\).

Next, we bound the number of queries in each phase.
In phase \(\ell\), writing \(r_\ell:=\dim(\operatorname{span}(\mathcal A_\ell))\le d\) and \(K_\ell:=\lvert\mathcal A_\ell\rvert\le K\), Lemma~\ref{lem:hp-design-estimates} shows that the estimator makes at most
\[
    Q_\ell =
    O\left(
    \frac{r_\ell}{\epsilon_\ell}\,
    \log\frac{K_\ell}{\delta'}\,
    \log\Big(4+\frac{r_\ell}{\epsilon_\ell}\Big)
    \log\log\Big(4+\frac{r_\ell}{\epsilon_\ell}\Big)
    \right)
\]

oracle queries deterministically. 
For every \(\ell\le L_T\) we have \(\epsilon_\ell=2^{-\ell}\ge1/(2T)\), so \(4+r_\ell/\epsilon_\ell=O(dT)\), and \(\log(K_\ell/\delta') \le\log\big(K(L_T+1)/\delta\big) =O\big(\log\tfrac{K\log T}{\delta}\big)\). 
Hence
\[
    Q_\ell =
    O\left(
    \frac{d}{\epsilon_\ell}\,
    \log\frac{K\log T}{\delta}\,
    \log(dT)\log\log(dT)
    \right).
\]

We now combine the two bounds. The regret incurred in phase \(\ell\) is at most the number of queries performed in that phase multiplied by \(8\epsilon_\ell\). 
This holds whether or not the phase completes, since the gap bound depends only on \(\mathcal A_\ell\), which is determined before the phase begins, and the number of queries performed in phase \(\ell\) never exceeds \(Q_\ell\). 
Hence, on \(\mathcal E\), the regret incurred during the planned phases is at most
\[
    \sum_{\ell=0}^{L_T}
    8\epsilon_\ell\,Q_\ell =
    O\left( (L_T+1)\,d\,
    \log\frac{K\log T}{\delta}\,
    \log(dT)\log\log(dT)
    \right).
\]
If all planned phases finish before \(T\) queries are made, then the current active set is \(\mathcal A_{L_T+1}\), and the remaining queries are to actions in this set. 
By the claim applied to phase \(L_T\), every action in \(\mathcal A_{L_T+1}\) has gap at most \(4\epsilon_{L_T}\le4/T\). Since there are at most \(T\) remaining queries, their total regret is at most \(4\), which is absorbed in the bound above. 
Therefore \(L_T+1=O(\log T)\) gives the high-probability bound.

Finally, we prove the expected regret bound. Run \LVGElim{} with confidence parameter \(\delta=1/T\). Since every gap is at most \(1\) and there are at most \(T\) queries, the regret is at most \(T\) deterministically. 
Hence, using the high-probability bound with \(\delta=1/T\),
\[
    \mathbb E[R(T)]
    \le O\left( d\log T\, \log(KT\log T)\, \log(dT)\log\log(dT) \right)
    + \Pr[\mathcal E^c]\,T.
\]
Here \(\Pr[\mathcal E^c]\le1/T\), so the second term is at most \(1\). Since \(\log(KT\log T)\le\log(KT^2)\le2\log(KT)\), we conclude
\[
\mathbb E[R(T)] = O\left(d\log T\,\log(KT)\,\log(dT)\log\log(dT)\right),
\]
which completes the proof.
\end{proof}

\begin{remark}[Bounded-variance rewards]
\label{rem:bounded-variance}
Our upper-bound analysis uses bounded rewards only through the quantum mean estimation primitive.
Under the bounded-variance assumption of \citet{wan2023quantum}, with a known variance bound \(\sigma^2\) and rewards of finite support size \(n\), the nondestructive median estimator and the unbiased mean estimator of \citet[Propositions~3.1 and~3.2]{cornelissen2023sublinear} can be combined to obtain estimates with bias \(O(\varepsilon\sigma)\), variance \(O(\sigma^2/t^2)\), and query complexity \(\widetilde O(t)\).
This suggests that the same design-based argument can be extended by scaling the confidence radii and query allocations with \(\sigma\), at the cost of additional polylogarithmic factors, including a logarithmic dependence on \(n\).
A complete treatment requires a corresponding deterministic-budget estimator and is left for future work.
\end{remark}

\paragraph{Comparison with a QMC variant.}
A natural alternative is to keep the same phased elimination algorithm and the same small-support designs, but estimate the support means using the QMC estimator from Lemma~\ref{lem:qmc}.
This gives a simpler design estimator with direct high-probability guarantees on the support.
The price is that the errors aggregate through worst-case absolute error: if each support estimate has error at most \(\rho\), then extrapolation through the prediction weights gives an error of at most \(\sqrt{C_G r}\,\rho\) for every active action by the Cauchy--Schwarz inequality.
Achieving accuracy \(\epsilon\) thus requires \(\rho\asymp\epsilon/\sqrt r\), which leads to a design-estimation cost of order \(r^{3/2}/\epsilon\) up to logarithmic factors.

The low-variance estimator from Lemma~\ref{lem:lvqme} avoids this loss because independent fluctuations aggregate through variance.
With a query allocation matched to the design weights, the variances add up to \(O(\epsilon^2)\) using the inequality \(\sum_{a\in S}\alpha_a(x)^2/\pi(a)\le C_G r\), without the \(\sqrt r\) inflation in the required support accuracy.

The following lemma gives the guarantee for the QMC design estimator.

\begin{lemma}[QMC design estimates]
\label{lem:qmc-design-estimates}
Fix an active set \(\mathcal A\subseteq\mathbb R^d\) and a design \((S,\pi,M)\) satisfying Lemma~\ref{lem:small-support-g-optimal-design}, and let
\(r=\dim(\operatorname{span}(\mathcal A))\ge1\).
Given an accuracy parameter \(\epsilon\in(0,1]\) and a confidence parameter \(\delta\in(0,1)\), set \(\rho:=\min\{1,\epsilon/\sqrt{C_G r}\}\).
For each \(a\in S\), let \(\widehat\mu_a\) be the output of an independent call \(\QMC(\mathcal O_a,\rho,\delta/\lvert S\rvert)\), and define \( \widehat\mu(x):=\sum_{a\in S}\alpha_a(x)\,\widehat\mu_a\) for all \(x\in\mathcal A\), where \(\{\alpha_a(x)\}_{a\in S}\) are the prediction weights from Lemma~\ref{lem:small-support-g-optimal-design}.
Then the total number of oracle queries is always bounded by \(O(r^{3/2}\epsilon^{-1}\log(r/\delta))\), and with probability at least \(1-\delta\),
\[
    \lvert\widehat\mu(x)-\mu(x)\rvert\le\epsilon\qquad\text{for all }x\in\mathcal A .
\]
\end{lemma}

\begin{proof}
By Lemma~\ref{lem:qmc} and a union bound over \(S\), with probability at least \(1-\delta\), the estimates \(\{\widehat\mu_a:a\in S\}\) satisfy \(\lvert\widehat\mu_a-\mu(a)\rvert\le\rho\) for every \(a\in S\). 
On this event, fix \(x\in\mathcal A\). Since Lemma~\ref{lem:small-support-g-optimal-design} gives \(x=\sum_{a\in S}\alpha_a(x)a\), the linear model implies \(\mu(x)=\sum_{a\in S}\alpha_a(x)\mu(a)\).
Therefore, by the Cauchy--Schwarz inequality weighted by \(\pi\),
\[
\lvert\widehat\mu(x)-\mu(x)\rvert =\Big\lvert
\sum_{a\in S}\alpha_a(x)\big(\widehat\mu_a-\mu(a)\big)
\Big\rvert
\le \left(\sum_{a\in S}\frac{\alpha_a(x)^2}{\pi(a)}\right)^{\frac{1}{2}}
\left(\sum_{a\in S}\pi(a)(\widehat\mu_a-\mu(a))^2\right)^{\frac{1}{2}}
\le \sqrt{C_G r}\,\rho
\le \epsilon.
\]

It remains to bound the number of queries.
Each QMC call uses \(O(\rho^{-1}\log(|S|/\delta))\) oracle queries by Lemma~\ref{lem:qmc}.
By the definition of \(\rho\), using \(r\ge1\) and \(\epsilon\le1\), we have \(\rho^{-1}\le1+\sqrt{C_Gr}/\epsilon\le(1+\sqrt{C_G})\sqrt r/\epsilon\).
With \(|S|\le c_sr\), the total number of queries over \(a\in S\) is \(O\big(|S|\,\rho^{-1}\log(|S|/\delta)\big)=O\big(r^{3/2}\epsilon^{-1}\log(r/\delta)\big)\).
\end{proof}

We denote this estimator by \(\QMCDesignEstimate\), and let \(\QMCGElim\) be the variant of Algorithm~\ref{alg:lv-g-elim} obtained by replacing \(\LVQMEDesignEstimate\) with \(\QMCDesignEstimate\) in Line~\ref{line:LVQMEDesignEstimate}.

\begin{corollary}[Regret of \QMCGElim]
\label{cor:qmc-g-elim-regret}
Assume \(T\ge2\) and \(dT\ge3\).
For any confidence parameter \(\delta\in(0,1)\), with probability at least \(1-\delta\), \QMCGElim satisfies \(R(T) = O\left( d^{3/2}\log T \log\frac{d\log T}{\delta} \right)\).
Moreover, running \QMCGElim with confidence parameter \(1/T\) gives \(\mathbb E[R(T)] = O\left( d^{3/2}\log T\, \log(dT) \right)\).
\end{corollary}

\begin{proof}
We follow the proof of Theorem~\ref{thm:lv-g-elim-regret}, replacing Lemma~\ref{lem:hp-design-estimates} by Lemma~\ref{lem:qmc-design-estimates}. The elimination argument and the instantaneous regret bound are unchanged.
The number of queries in phase \(\ell\) is always bounded by
\[
Q_\ell = O\left(\frac{r_\ell^{3/2}}{\epsilon_\ell}\log\frac{r_\ell}{\delta'}\right)
=O\left(\frac{d^{3/2}}{\epsilon_\ell}\log\frac{d\log T}{\delta}\right),
\]
where \(\delta'=\delta/(L_T+1)\). Hence the regret incurred in phase \(\ell\) is at most \(8\epsilon_\ell Q_\ell =O\big(d^{3/2}\log\frac{d\log T}{\delta}\big)\), and summing over the \(O(\log T)\) phases gives the high-probability bound.

For the expected regret bound, we run \QMCGElim with confidence parameter \(\delta=1/T\). As in the proof of Theorem~\ref{thm:lv-g-elim-regret}, the failure event contributes at most \(1\) to the expected regret, and \(\log(dT\log T)=O(\log(dT))\). This gives the stated bound.
\end{proof}

The two estimators thus offer a trade-off.
The QMC variant avoids the \(\log K\) factor because its high-probability guarantee requires a union bound only over the design support, but it pays the \(\sqrt d\) loss from worst-case error propagation through the prediction weights.
By contrast, the single-run design estimate underlying \(\LVQMEDesignEstimate\) is accurate for each fixed action only with constant probability, and the median amplification needed to cover all active actions is the source of the \(\log K\) factor in Theorem~\ref{thm:lv-g-elim-regret}.
When \(K=\operatorname{poly}(d)\), this logarithmic dependence is mild, and the regret of \LVGElim is nearly linear in \(d\).
For exponentially large action sets with \(\log K=\Theta(d)\), \QMCGElim can give a better dimension dependence.

\section{Conclusion}
\label{sec:conclusion}
We proved the first regret lower bounds for quantum bandits in the reward oracle model of \citet{wan2023quantum}: \(\Omega(K\log(1+T/K))\) for multi-armed bandits and \(\Omega(d\log(1+T/d))\) for linear bandits with only \(d\) actions.
For QMAB, we also obtained a matching upper bound, establishing the minimax expected regret \(\Theta(K\log(1+T/K))\).
For finite-action QLB, our algorithm \LVGElim attains regret nearly linear in \(d\) when \(K=\operatorname{poly}(d)\), improving the dimension dependence of prior upper bounds and matching our lower bound up to polylogarithmic factors.

Several questions remain open.
For general action sets, the optimal dependence on \(d\) is unresolved.
The upper bounds are \(O(d^2\,\mathrm{polylog}\,T)\) from \citet{wan2023quantum} and \(O(d^{3/2}\,\mathrm{polylog}\,T)\) from a linear-kernel specialization of the analysis of \citet{hikima2024quantum} (see Appendix~\ref{app:hikima-linear}), while our lower bound is \(\Omega(d\log(1+T/d))\).
For finite action sets with \(K\) superpolynomial in \(d\), the nearly linear bound of \LVGElim{} carries an extra \(\log K\) factor, while the QMC variant removes this dependence at the cost of a \(d^{3/2}\) dimension factor.
Whether the two advantages can be combined, for instance by improving the \(\log K\) dependence toward the \(\sqrt{\log K}\) dependence in classical finite-action linear bandits, remains open.
On the lower bound side, our construction uses only \(K=d\) actions and does not rule out a mild \(K\) dependence for larger fixed action sets, which would have to remain consistent with the QMC upper bound.
Beyond these gaps, our upper bound may extend to the bounded-variance assumption via the estimator replacement sketched in Remark~\ref{rem:bounded-variance}, and we leave a full treatment to future work.

\section*{Acknowledgments}
The authors thank the anonymous reviewers of SODA 2027 for the helpful comments and suggestions.
The work of Zhuohua Li was supported in part by the Xiaomi Young Talent Program and by the Fundamental Research Funds for the Central Universities under Grant QTZX26133.
The work of John C.S. Lui was supported in part by the RGC SRFS2122-4S02 and GRF-14202925.

\newpage

\appendix
\section{Proof of Lemma~\ref{lem:qmc-tail}}
\label{app:qmc-tail}

\begin{proof}
Fix \(\epsilon\in(0,1]\) and \(\delta\in(0,1/2]\).
Let \(h\) be the smallest odd integer satisfying \(2^{-h}\le\delta\), so that \(h\ge1\).
Let \(N\) denote the dimension of the Fourier register used in each amplitude estimation run.
We choose \(N\) to be the smallest power of two satisfying \(N\ge64\pi/\epsilon\), so that \(64\pi/\epsilon\le N<128\pi/\epsilon\).
The estimator \(\QMC(\mathcal O,\epsilon,\delta)\) returns the median \(\widehat\mu\) of \(h\) independent copies of \(Y\in[0,1]\), where \(Y\) is the output of one amplitude estimation run with parameter \(N\).

By \citet[Theorem~12]{brassard2000quantum}, for every integer \(j\ge2\), we have
\[
    \Pr\left[
        |Y-\mu|>
        \frac{2\pi j\sqrt{\mu(1-\mu)}}{N}
        +\frac{\pi^2j^2}{N^2}
    \right]
    \le \frac{1}{2(j-1)}.
\]

For \(\epsilon\le \gamma\le1\), take \(j=\lfloor N\gamma/(4\pi)\rfloor\).
Since \(N\ge64\pi/\epsilon\) and \(\gamma\ge\epsilon\), we have \(N\gamma/(4\pi)\ge16\), so \(j\ge2\).
We first have 
\[
\frac{2\pi j\sqrt{\mu(1-\mu)}}{N}+\frac{\pi^2j^2}{N^2}\le \frac{\gamma}{4} + \frac{\gamma^2}{16}<\gamma,
\]
where we used \(j\le N\gamma/(4\pi)\), \(\sqrt{\mu(1-\mu)}\le1/2\), and \(\gamma^2\le \gamma\).

Since \(N\gamma/(8\pi)\ge8>2\), we have \(j-1\ge N\gamma/(4\pi)-2\ge N\gamma/(4\pi)-N\gamma/(8\pi)=N\gamma/(8\pi)\).
Thus,
\[\Pr[|Y-\mu|\ge \gamma]\le\frac{1}{2(j-1)} \le \frac{4\pi}{N\gamma}\le\frac{\epsilon}{16\gamma}.\]
For \(\gamma>1\), we have \(\Pr[|Y-\mu|\ge \gamma]=0\) because \(Y,\mu\in[0,1]\).
Thus \(\Pr[|Y-\mu|\ge \gamma]\le\epsilon/(16\gamma)\) holds for all \(\gamma\ge\epsilon\).

Applying the median lemma of~\citet[Lemma~1]{nagaj2009fast}
to the \(h\) independent outputs gives
\[
\begin{aligned}
    \Pr[|\widehat\mu-\mu|\ge\gamma] \le
    2^h\left(\frac{\epsilon}{16\gamma}\right)^{\frac{h}{2}}
    = 2^{-h}\left(\frac{\epsilon}{\gamma}\right)^{\frac{h}{2}}
    \le \delta\left(\frac{\epsilon}{\gamma}\right)^{\frac12\log(1/\delta)},
\end{aligned}
\]
where the last inequality uses \(2^{-h}\le\delta\),
\(h\ge\log(1/\delta)\), and
\(\epsilon/\gamma\le1\).

By the minimality of \(h\), we have \(2^{-(h-2)}>\delta\) and thus \(h<\log_2(1/\delta)+2\).
Since \(\delta\le1/2\), we have \(2\le2\log_2(1/\delta)\), and hence \(h<3\log_2(1/\delta)=O(\log(1/\delta))\).
Moreover, each amplitude estimation run uses one state preparation and \(N-1\) controlled Grover iterations, requiring \(1+2(N-1)=2N-1\) queries to \(\mathcal O\) and \(\mathcal O^\dagger\) in total by Remark~\ref{rem:controlled}.
Since \(N<128\pi/\epsilon\), the total number of queries satisfies
\[
    h(2N-1)<2hN<\frac{256\pi h}{\epsilon}
    =O\left(\frac{1}{\epsilon}\log\frac{1}{\delta}\right),
\]
proving the query bound for a universal constant \(C_2>1\).
\end{proof}

\section{Small-Support Approximate \ensuremath{G}-Optimal Designs}
\label{app:g-optimal-design}

This appendix gives the details for Lemma~\ref{lem:small-support-g-optimal-design}. The proof uses two standard tools from optimal experimental design: the Kiefer--Wolfowitz equivalence theorem of \citet{kiefer1960equivalence} and the rounding theorem and continuous relaxation method of \citet{allen2021near}.

\begin{theorem}[{\citealp{kiefer1960equivalence}}]
\label{thm:kw-g-optimal}
Let \(\mathcal Z\subseteq\mathbb R^{r}\) be finite with \(\operatorname{span}(\mathcal Z)=\mathbb R^{r}\). 
Then there exists a distribution \(w^\star\in\Delta(\mathcal Z)\) such that \(M_{w^\star}=\sum_{z\in\mathcal Z}w^\star(z)zz^\top\) is positive definite and \(\max_{z\in\mathcal Z}z^\top M_{w^\star}^{-1}z\le r\).
\end{theorem}

\begin{theorem}[{\citealp[Theorem~2.1]{allen2021near}}]
\label{thm:az-g-optimal-rounding}
Let \(z_1,\ldots,z_N\in\mathbb R^r\), and for \(s\in\mathbb R_{\ge0}^N\) write \(\Sigma_s=\sum_{i=1}^N s_i z_i z_i^\top\). 
Suppose \(\varepsilon\in(0,1/6]\) and \(N\ge m\ge 5r/\varepsilon^2\). Let \(u\in[0,1]^N\) satisfy \(\sum_i u_i\le m\) and \(\Sigma_u\succ0\). 
Then, in time \(\widetilde O(Nr^2)\), one can compute \(\widehat s\in\{0,1\}^N\) with \(\sum_i\widehat s_i\le m\) such that
\(\Sigma_{\widehat s}\succ0\) and \(\max_{i\in[N]} z_i^\top\Sigma_{\widehat s}^{-1}z_i \le (1+6\varepsilon)\max_{i\in[N]} z_i^\top\Sigma_u^{-1}z_i\).
\end{theorem}

In the notation of \citet{allen2021near}, Theorem~\ref{thm:az-g-optimal-rounding} is their Theorem~2.1 applied with multiplicity bound \(b=1\) and the \(G\)-optimality criterion \(f_G(\Sigma)=\max_{i\in[N]}z_i^\top\Sigma^{-1}z_i\), which satisfies the monotonicity and reciprocal sub-linearity assumptions required there \citep[Fact~1.3]{allen2021near}.
We also use the continuous relaxation method of \citet[Section~3]{allen2021near} to compute a constant-factor approximate continuous \(G\)-optimal design.

\begin{proof}[Proof of Lemma~\ref{lem:small-support-g-optimal-design}]
Let \(V=\operatorname{span}(\mathcal X)\). Choose an orthonormal basis \(U\in\mathbb R^{d\times r}\) for \(V\), and write \(\tilde x=U^\top x\in\mathbb R^r\) for each \(x\in\mathcal X\). Since \(U^\top\) is an isometry on \(V\), the map \(x\mapsto \tilde x\) is a bijection from \(\mathcal X\) onto \(\widetilde{\mathcal X}:=\{\tilde x:x\in\mathcal X\}\), and \(\widetilde{\mathcal X}\) spans \(\mathbb R^r\). 
We first construct the design in these coordinates.

Fix a universal constant \(\varepsilon_0\in(0,1/6]\) and set \(m=\lceil 5r/\varepsilon_0^2\rceil\le c_sr\), where \(c_s:=5/\varepsilon_0^2+1\), using \(r\ge1\). 
We compute an approximately \(G\)-optimal distribution on \(\widetilde{\mathcal X}\). Consider the continuous design problem of mass \(m\) on \(\widetilde{\mathcal X}\), namely minimizing \(f_G(\Sigma_u)=\max_{\tilde x\in\widetilde{\mathcal X}} \tilde x^\top\Sigma_u^{-1}\tilde x\) over fractional designs \(u\ge0\) of total mass \(m\).
By Theorem~\ref{thm:kw-g-optimal}, there exists a distribution \(w^\star\) on \(\widetilde{\mathcal X}\) with \(\max_{\tilde x}\tilde x^\top M_{w^\star}^{-1}\tilde x\le r\); since \(f_G(t\Sigma)=t^{-1}f_G(\Sigma)\), the scaled design \(mw^\star\) is feasible and shows that the optimal value of the mass-\(m\) problem is at most \(r/m\).
The criterion \(f_G\) satisfies the convexity and Lipschitz assumptions of \citet[Remark~3.4]{allen2021near}, so the entropic
mirror descent method of \citet[Corollary~3.9]{allen2021near}, applied with budget \(k=m\), multiplicity bound \(b=m\), and relative error \(\delta=1/2\), computes a fractional design \(u_0\) of total mass \(m\) with \(f_G(\Sigma_{u_0})\le\tfrac32\,r/m\).
Writing \(u_0=mw\) for a distribution \(w\) on \(\widetilde{\mathcal X}\), the matrix \(\widetilde M_w=\sum_{\tilde x}w(\tilde x)\tilde x\tilde x^\top\) satisfies \(\max_{\tilde x\in\widetilde{\mathcal X}} \tilde x^\top\widetilde M_w^{-1}\tilde x\le C_0r\) with \(C_0:=\tfrac32\); in particular, \(\widetilde M_w\) is positive definite, since its \(G\)-value is finite and \(\widetilde{\mathcal X}\) spans \(\mathbb R^r\).

Following the standard reduction from multiplicity \(m\) to the \(b=1\) case in \citet{allen2021near}, replace each point \(\tilde x\in\widetilde{\mathcal X}\) by \(m\) identical copies.
The copied list has \(N=m|\widetilde{\mathcal X}|\ge m\) points. 
On this copied list, let \(u'\) be the fractional design assigning weight \(w(\tilde x)\) to each copy of \(\tilde x\). Then \(u'\) has all coordinates in \([0,1]\),
total mass \(m\), and information matrix \(\Sigma_{u'}=m\widetilde M_w\succ0\), so its \(G\)-value is at most \(C_0r/m\). 
Applying Theorem~\ref{thm:az-g-optimal-rounding} with budget \(m\) and \(\varepsilon=\varepsilon_0\), we obtain an integral design \(\widehat s\) on the copied list with total mass \(\widehat m\le m\), positive definite information matrix, and \(G\)-value at most \(C_1r/m\), where \(C_1:=(1+6\varepsilon_0)C_0\).

Merge identical copies back into the points of \(\widetilde{\mathcal X}\). 
Let \(\widehat s_{\tilde x}\) be the resulting multiplicity of \(\tilde x\), let \(\widetilde S=\{\tilde x:\widehat s_{\tilde x}>0\}\), and define
\(\widetilde\pi(\tilde x)=\widehat s_{\tilde x}/\widehat m\) for \(\tilde x\in\widetilde S\). Then
\(|\widetilde S|\le \widehat m\le m\le c_s r\). 
With \(\widetilde A_{\widehat s} =\sum_{\tilde x}\widehat s_{\tilde x}\tilde x\tilde x^\top\succ0\), the normalized matrix
\(\widetilde M =\sum_{\tilde x\in\widetilde S}\widetilde\pi(\tilde x) \tilde x\tilde x^\top =\widetilde A_{\widehat s}/\widehat m\)
satisfies
\[
\max_{\tilde x\in\widetilde{\mathcal X}} \tilde x^\top\widetilde M^{-1}\tilde x
= \widehat m \max_{\tilde x\in\widetilde{\mathcal X}} \tilde x^\top\widetilde A_{\widehat s}^{-1}\tilde x
\le \widehat m\cdot\frac{C_1r}{m}
\le C_1r .
\]

Now map the design back to the original space. 
Let \(S\subseteq\mathcal X\) be the preimage of \(\widetilde S\) under \(x\mapsto\tilde x\), and set \(\pi(x)=\widetilde\pi(U^\top x)\) for \(x\in S\). 
Then \(|S|=|\widetilde S|\le c_s r\). The original information matrix is \(M=U\widetilde M U^\top\). 
Since \(\widetilde M\succ0\), we have \(\operatorname{Range}(M)=V=\operatorname{span}(\mathcal X)\) and \(M^+=U\widetilde M^{-1}U^\top\). 
Hence
\[
    \max_{x\in\mathcal X}x^\top M^+x
    =
    \max_{\tilde x\in\widetilde{\mathcal X}}
    \tilde x^\top\widetilde M^{-1}\tilde x
    \le C_1r .
\]
This proves the design bound with \(C_G=C_1\).

It remains to prove the identities for the prediction weights. 
For \(\alpha_a(x)=\pi(a)a^\top M^+x\), since \(x\in\operatorname{Range}(M)\),
\[
    \sum_{a\in S}\alpha_a(x)a = \sum_{a\in S}\pi(a)aa^\top M^+x = MM^+x = x.
\]
Moreover,
\[
\sum_{a\in S}\frac{\alpha_a(x)^2}{\pi(a)} = \sum_{a\in S}\pi(a)(a^\top M^+x)^2 = x^\top M^+MM^+x = x^\top M^+x.
\]
The bound \(x^\top M^+x\le C_G r\) follows from the design bound above.

Finally, we account for the running time.
The continuous design problem is convex, and a constant-factor approximate solution suffices for the proof above, which can be computed in polynomial time by the ellipsoid method, as noted by \citet[Section~3]{allen2021near}.
The rounding step runs in time \(\widetilde O(Nr^2)=\widetilde O(|\mathcal X|\,r^3)\) by Theorem~\ref{thm:az-g-optimal-rounding}, since the copied list has \(N=m|\mathcal X|=O(r|\mathcal X|)\) points.
The remaining steps, namely the basis computation, copying, merging, and normalization, take time polynomial in \(|\mathcal X|\) and \(d\).
Hence the design of Lemma~\ref{lem:small-support-g-optimal-design} is computable in time polynomial in \(|\mathcal X|\) and \(d\).
\end{proof}

\section{An \ensuremath{O(d\sp{3/2}\,\mathrm{polylog}\,T)} Bound for General Action Sets}
\label{app:hikima-linear}

In this appendix, we specialize the analysis of \citet{hikima2024quantum} to the \(d\)-dimensional linear kernel.
Although they do not state a separate result for quantum linear bandits, their confidence bound, combined with a determinant-based stage-count argument, yields an \(O(d^{3/2}\,\mathrm{polylog}\,T)\) regret bound for general, possibly infinite, action sets.

Throughout, \(\mathcal A\subseteq\{x\in\mathbb R^d:\|x\|_2\le1\}\) is an arbitrary action set, \(\|\theta\|_2\le1\), and \(\mu(x)=x^\top\theta\in[0,1]\), as in Section~\ref{sec:prelim}, and we assume \(T\ge2\).
The linear kernel \(k(x,z)=x^\top z\) has feature map \(\phi(x)=x\) and reproducing kernel Hilbert space \(\mathbb R^d\), so the boundedness assumptions of \citet{hikima2024quantum} hold with \(\sup_{x,z}|k(x,z)|\le1\) and \(\|\theta\|_{\mathcal H_k}=\|\theta\|_2\le1\).
We run their algorithm \textsc{QMCKernelUCB} with regularization parameter \(\rho=1\) and tradeoff parameter \(\eta=1\), and assume that an action maximizing the UCB index below is attained in each stage (e.g., \(\mathcal A\) compact).
Otherwise, a \(1/T\)-approximate maximizer changes the bounds by at most an additive constant.
If \(\mathcal A=\{0\}\), the regret is identically zero.
Otherwise, removing the zero vector does not change the optimal mean, since all mean rewards lie in \([0,1]\).
Thus we may assume \(0\notin\mathcal A\), and hence \(\varepsilon_s>0\) in every stage below.

For the linear kernel, the kernel expressions of \citet[Proposition~4.1]{hikima2024quantum} reduce to weighted ridge regression, and the algorithm takes the same weighted ridge-regression form as QLinUCB of \citet{wan2023quantum} but uses a smaller confidence radius.
Given \(T\), a failure probability \(\delta\in(0,1)\), and a stage upper bound \(M\), the algorithm maintains \(V_s=I_d+\sum_{k=1}^{s}\varepsilon_k^{-2}x_kx_k^\top\) and \(\widehat\theta_s=V_s^{-1}\sum_{k=1}^{s}\varepsilon_k^{-2}y_kx_k\).
In each stage \(s=1,2,\ldots\), terminating at round \(T\), it selects \(x_s\in\argmax_{x\in\mathcal A}\bigl(x^\top\widehat\theta_{s-1}+\beta_{s-1}\|x\|_{V_{s-1}^{-1}}\bigr)\) with \(\beta_{s-1}:=1+\sqrt{s-1}\), sets \(\varepsilon_s:=\|x_s\|_{V_{s-1}^{-1}}\le\|x_s\|_2\le1\), and plays \(x_s\) for the next \(q_s:=\lceil\frac{C_1}{\varepsilon_s}\log\frac M\delta\rceil\) rounds, which suffice for the at most \(q_s\) oracle queries of \(\QMC(\mathcal O_{x_s},\varepsilon_s,\delta/M)\) in Lemma~\ref{lem:qmc}, producing an estimate \(y_s\) of \(x_s^\top\theta\) when the stage completes.
Every stage thus occupies exactly \(q_s\) consecutive rounds, with \(x_s\) selected for the entire block even if the \(\QMC\) call finishes early, unless the horizon is exhausted during the block, in which case the stage remains incomplete and \(V_s\) is not updated.
The estimator uses only the oracles \(\mathcal O_{x_s}\) and \(\mathcal O_{x_s}^\dagger\), by the accounting of Remark~\ref{rem:controlled}.

The first ingredient is the confidence bound of \citet{hikima2024quantum}, specialized to this setting.
It improves the width \(1+\sqrt{ds}\) in the corresponding step of \citet[Lemma~3]{wan2023quantum} to \(1+\sqrt s\), which is the source of the improvement from \(d^2\) to \(d^{3/2}\).

\begin{lemma}[{Confidence bound; from \citealp[Proposition~4.2]{hikima2024quantum} with \(\rho=\eta=S=1\)}]
\label{lem:hikima-confidence}
Let \(m\) be the total number of stages initiated by the algorithm and suppose \(M\ge m\).
With probability at least \(1-\delta\), for every \(s\in\{0,1,\ldots,m-1\}\) and every \(x\in\mathcal A\), we have \(\lvert x^\top\theta-x^\top\widehat\theta_s\rvert\le\beta_s\|x\|_{V_s^{-1}}\), where \(\beta_s=1+\sqrt s\).
\end{lemma}

Here the case \(s=0\) holds deterministically, since \(\widehat\theta_0=0\) and \(\lvert x^\top\theta\rvert\le\|x\|_2\|\theta\|_2\le\|x\|_{V_0^{-1}}\) with \(V_0=I_d\).

The second ingredient is a deterministic bound on the number of stages.

\begin{lemma}[Stage count]
\label{lem:hikima-stage-count}
Let \(m^\star:=\lceil d\log_2(1+T^2/d)\rceil\).
For any \(M\ge3\) and \(\delta\in(0,1)\), at most \(m^\star\) stages are completed, and at most \(m^\star+1\) stages are initiated before the horizon is exhausted.
\end{lemma}

\begin{proof}
By the choice of \(\varepsilon_s\) and the matrix determinant lemma, each completed stage doubles the determinant, since \(\det(V_s)=\det(V_{s-1})\bigl(1+\varepsilon_s^{-2}\|x_s\|_{V_{s-1}^{-1}}^2\bigr)=2\det(V_{s-1})\), so after \(m\) completed stages \(\det(V_m)=2^m\).
Applying the arithmetic-geometric mean inequality to the eigenvalues of \(V_m\), and using \(\|x_k\|_2\le1\), we obtain \(d\,2^{m/d}\le\operatorname{tr}(V_m)\le d+\sum_{k=1}^{m}\varepsilon_k^{-2}\).
On the other hand, each completed stage \(k\) occupies exactly \(q_k\) rounds within the horizon, so \(T\ge\sum_{k=1}^{m}q_k\ge\sum_{k=1}^{m}\varepsilon_k^{-1}\ge\bigl(\sum_{k=1}^{m}\varepsilon_k^{-2}\bigr)^{1/2}\), using \(C_1\log(M/\delta)\ge1\), which holds since \(M\ge3\) and \(\delta<1\), and \(\sum_ka_k\ge\bigl(\sum_ka_k^2\bigr)^{1/2}\) for \(a_k\ge0\).
Hence \(\sum_{k=1}^{m}\varepsilon_k^{-2}\le T^2\), which is \citet[Lemma~C.3]{hikima2024quantum} specialized to \(\eta=1\).
Combining the two bounds gives \(2^{m/d}\le1+T^2/d\), so at most \(m^\star\) stages are completed.
At most one further stage is in progress when the horizon is exhausted, and hence at most \(m^\star+1\) stages are initiated.
\end{proof}

\begin{proposition}[General action sets]
\label{prop:hikima-linear}
Run the algorithm above with \(M:=\max\{3,\,m^\star+1\}\).
With probability at least \(1-\delta\), its regret satisfies \(R(T)=O\bigl(d^{3/2}\log^{3/2}(T)\log\frac{dT}{\delta}\bigr)\).
Taking \(\delta=1/T\) gives \(\mathbb E[R(T)]=O\bigl(d^{3/2}\log^{3/2}(T)\log(dT)\bigr)\).
\end{proposition}

\begin{proof}
By Lemma~\ref{lem:hikima-stage-count}, the number \(m\) of initiated stages satisfies \(m\le m^\star+1\le M\), so Lemma~\ref{lem:hikima-confidence} applies, and we condition on its event.
Fix an initiated stage \(s\) and let \(x^\star\in\argmax_{x\in\mathcal A}x^\top\theta\).
By Lemma~\ref{lem:hikima-confidence} applied at stage \(s-1\) to \(x^\star\), the selection rule, and Lemma~\ref{lem:hikima-confidence} applied at stage \(s-1\) to \(x_s\), we have \((x^\star)^\top\theta\le(x^\star)^\top\widehat\theta_{s-1}+\beta_{s-1}\|x^\star\|_{V_{s-1}^{-1}}\le x_s^\top\widehat\theta_{s-1}+\beta_{s-1}\varepsilon_s\le x_s^\top\theta+2\beta_{s-1}\varepsilon_s\), so every round of stage \(s\) incurs instantaneous regret at most \(2\beta_{s-1}\varepsilon_s\).
Stage \(s\) occupies at most \(q_s\le\frac{C_1}{\varepsilon_s}\log\frac M\delta+1\) rounds, so the regret incurred in stage \(s\) is at most \(2C_1\beta_{s-1}\log\frac M\delta+2\beta_{s-1}\varepsilon_s\).
There are at most \(M\) initiated stages, and \(\beta_{s-1}\le1+\sqrt M\), so summing over stages gives \(R(T)=O\bigl(M^{3/2}\log\frac M\delta\bigr)\) on this event, which is the assembly of \citet[Proposition~5.2]{hikima2024quantum} with \(\eta=1\).
Since \(M=O(d\log T)\) and \(\log(M/\delta)=O(\log\frac{dT}\delta)\), the high probability bound follows.
Finally, take \(\delta=1/T\).
Since \(\mu(\mathcal A)\subseteq[0,1]\), each round incurs regret at most \(1\) and \(R(T)\le T\) deterministically, so the failure event contributes at most \(\delta T=1\) to the expected regret, and \(\log(MT)=O(\log(dT))\) gives the expected bound.
\end{proof}

\bibliography{reference}

@inproceedings{wang2021quantum,
  title={Quantum exploration algorithms for multi-armed bandits},
  author={Wang, Daochen and You, Xuchen and Li, Tongyang and Childs, Andrew M},
  booktitle={Proceedings of the AAAI Conference on Artificial Intelligence},
  volume={35},
  number={11},
  pages={10102--10110},
  year={2021}
}

@inproceedings{wan2023quantum,
  title     = {Quantum multi-armed bandits and stochastic linear bandits enjoy logarithmic regrets},
  author    = {Wan, Zongqi and Zhang, Zhijie and Li, Tongyang and Zhang, Jialin and Sun, Xiaoming},
  booktitle = {Proceedings of the AAAI Conference on Artificial Intelligence},
  volume    = {37},
  number    = {8},
  pages     = {10087--10094},
  year      = {2023}
}

@article{kiefer1960equivalence,
  title={The equivalence of two extremum problems},
  author={Kiefer, Jack and Wolfowitz, Jacob},
  journal={Canadian Journal of Mathematics},
  volume={12},
  pages={363--366},
  year={1960},
  publisher={Cambridge University Press}
}

@inproceedings{wang2025best,
  title = {Best Arm Identification with Quantum Oracles},
  author = {Wang, Xuchuang and Chen, Yu-Zhen Janice and Guedes de Andrade, Matheus and Allcock, Jonathan and Hajiesmaili, Mohammad and Lui, John C.S. and Towsley, Don},
  booktitle = {The 39th Annual AAAI Conference on Artificial Intelligence},
  year = {2025}
}

@article{montanaro2015quantum,
  title={Quantum speedup of {Monte Carlo} methods},
  author={Montanaro, Ashley},
  journal={Proceedings of the Royal Society A: Mathematical, Physical and Engineering Sciences},
  volume={471},
  number={2181},
  pages={20150301},
  year={2015},
  publisher={The Royal Society}
}

@article{beals2001quantum,
  title={Quantum lower bounds by polynomials},
  author={Beals, Robert and Buhrman, Harry and Cleve, Richard and Mosca, Michele and De Wolf, Ronald},
  journal={Journal of the ACM (JACM)},
  volume={48},
  number={4},
  pages={778--797},
  year={2001},
  publisher={ACM New York, NY, USA}
}

@article{ganzburg2012remez,
  title={On a Remez-type inequality for trigonometric polynomials},
  author={Ganzburg, Michael I},
  journal={Journal of Approximation Theory},
  volume={164},
  number={9},
  pages={1233--1237},
  year={2012},
  publisher={Elsevier}
}

@article{mande2026tight,
  title = {Tight {B}ounds for {Q}uantum {P}hase {E}stimation and {R}elated {P}roblems},
  author = {Mande, Nikhil S. and de Wolf, Ronald},
  journal = {{Quantum}},
  issn = {2521-327X},
  publisher = {{Verein zur F{\"{o}}rderung des Open Access Publizierens in den Quantenwissenschaften}},
  volume = {10},
  pages = {2140},
  month = jun,
  year = {2026},
}

@inproceedings{cornelissen2023sublinear,
  title={A sublinear-time quantum algorithm for approximating partition functions},
  author={Cornelissen, Arjan and Hamoudi, Yassine},
  booktitle={Proceedings of the 2023 annual ACM-Siam symposium on discrete algorithms (SODA)},
  pages={1245--1264},
  year={2023},
  organization={SIAM}
}

@article{wu2023quantum,
  title={Quantum heavy-tailed bandits},
  author={Wu, Yulian and Guan, Chaowen and Aggarwal, Vaneet and Wang, Di},
  journal={arXiv preprint arXiv:2301.09680},
  year={2023}
}

@article{dai2023quantum,
  title={Quantum bayesian optimization},
  author={Dai, Zhongxiang and Lau, Gregory Kang Ruey and Verma, Arun and Shu, Yao and Low, Bryan Kian Hsiang and Jaillet, Patrick},
  journal={Advances in neural information processing systems},
  volume={36},
  pages={20179--20207},
  year={2023}
}

@inproceedings{hikima2024quantum,
  title={Quantum kernelized bandits},
  author={Hikima, Yasunari and Murao, Kazunori and Takemori, Sho and Umeda, Yuhei},
  booktitle={The 40th Conference on Uncertainty in Artificial Intelligence},
  year={2024}
}

@inproceedings{yi2026quantum,
  title={Quantum {Lipschitz} bandits},
  author={Yi, Bongsoo and Kang, Yue and Li, Yao},
  booktitle={Proceedings of the AAAI Conference on Artificial Intelligence},
  volume={40},
  number={33},
  pages={27844--27851},
  year={2026}
}

@article{allen2021near,
  title={Near-optimal discrete optimization for experimental design: A regret minimization approach},
  author={Allen-Zhu, Zeyuan and Li, Yuanzhi and Singh, Aarti and Wang, Yining},
  journal={Mathematical Programming},
  volume={186},
  number={1},
  pages={439--478},
  year={2021},
  publisher={Springer}
}

@article{acin2001statistical,
  title={Statistical distinguishability between unitary operations},
  author={Acin, Antonio},
  journal={Physical review letters},
  volume={87},
  number={17},
  pages={177901},
  year={2001},
  publisher={APS}
}

@article{duan2007entanglement,
  title={Entanglement is not necessary for perfect discrimination between unitary operations},
  author={Duan, Runyao and Feng, Yuan and Ying, Mingsheng},
  journal={Physical review letters},
  volume={98},
  number={10},
  pages={100503},
  year={2007},
  publisher={APS}
}

@book{nielsen2010quantum,
  title={Quantum computation and quantum information},
  author={Nielsen, Michael A and Chuang, Isaac L},
  year={2010},
  publisher={Cambridge university press}
}

@book{borwein1995polynomials,
  title={Polynomials and Polynomial Inequalities},
  author={Borwein, Peter and Erdelyi, Tamas},
  volume={161},
  year={1995},
  publisher={Springer Science \& Business Media}
}

@book{lattimore2020bandit,
  title={Bandit algorithms},
  author={Lattimore, Tor and Szepesv{\'a}ri, Csaba},
  year={2020},
  publisher={Cambridge University Press}
}

@inproceedings{lattimore2020learning,
  title={Learning with good feature representations in bandits and in {RL} with a generative model},
  author={Lattimore, Tor and Szepesvari, Csaba and Weisz, Gellert},
  booktitle={International conference on machine learning},
  pages={5662--5670},
  year={2020},
  organization={PMLR}
}

@article{auer2002nonstochastic,
  title={The nonstochastic multiarmed bandit problem},
  author={Auer, Peter and Cesa-Bianchi, Nicolo and Freund, Yoav and Schapire, Robert E},
  journal={SIAM journal on computing},
  volume={32},
  number={1},
  pages={48--77},
  year={2002},
  publisher={SIAM}
}

@inproceedings{audibert2009minimax,
  title={Minimax policies for adversarial and stochastic bandits},
  author={Audibert, Jean-Yves and Bubeck, S{\'e}bastien},
  booktitle={Colt},
  pages={217--226},
  year={2009}
}

@article{abbasi2011improved,
  title={Improved algorithms for linear stochastic bandits},
  author={Abbasi-Yadkori, Yasin and P{\'a}l, D{\'a}vid and Szepesv{\'a}ri, Csaba},
  journal={Advances in neural information processing systems},
  volume={24},
  year={2011}
}

@inproceedings{dani2008stochastic,
  title={Stochastic linear optimization under bandit feedback},
  author={Dani, Varsha and Hayes, Thomas P and Kakade, Sham M},
  booktitle={21st Annual Conference on Learning Theory},
  number={101},
  pages={355--366},
  year={2008}
}

@inproceedings{li2019nearly,
  title={Nearly minimax-optimal regret for linearly parameterized bandits},
  author={Li, Yingkai and Wang, Yining and Zhou, Yuan},
  booktitle={Conference on Learning Theory},
  pages={2173--2174},
  year={2019},
  organization={PMLR}
}

@article{lai1985asymptotically,
  title={Asymptotically efficient adaptive allocation rules},
  author={Lai, Tze Leung and Robbins, Herbert},
  journal={Advances in applied mathematics},
  volume={6},
  number={1},
  pages={4--22},
  year={1985},
  publisher={Academic Press, Inc. Orlando, FL, USA}
}

@article{robbins1952some,
  title = {Some aspects of the sequential design of experiments},
  volume = {58},
  ISSN = {0273-0979},
  number = {5},
  journal = {Bulletin of the American Mathematical Society},
  publisher = {American Mathematical Society (AMS)},
  author = {Robbins,  Herbert},
  year = {1952},
  pages = {527–535}
}

@article{buchholz2025multi,
  title={Multi-Armed Bandits and Quantum Channel Oracles},
  author={Buchholz, Simon and K{\"u}bler, Jonas M and Sch{\"o}lkopf, Bernhard},
  journal={Quantum},
  volume={9},
  pages={1672},
  year={2025},
  publisher={Verein zur F{\"o}rderung des Open Access Publizierens in den Quantenwissenschaften}
}

@article{su2025quantum,
  title={Quantum Algorithms for Bandits with Knapsacks with Improved Regret and Time Complexities},
  author={Su, Yuexin and Yang, Ziyi and Huang, Peiyuan and Li, Tongyang and Ye, Yinyu},
  journal={arXiv preprint arXiv:2507.04438},
  year={2025}
}

@article{lumbreras2022multi,
  title={Multi-armed quantum bandits: Exploration versus exploitation when learning properties of quantum states},
  author={Lumbreras, Josep and Haapasalo, Erkka and Tomamichel, Marco},
  journal={Quantum},
  volume={6},
  pages={749},
  year={2022},
  publisher={Verein zur F{\"o}rderung des Open Access Publizierens in den Quantenwissenschaften}
}

@article{casale2020quantum,
  title={Quantum bandits},
  author={Casal{\'e}, Balthazar and Di Molfetta, Giuseppe and Kadri, Hachem and Ralaivola, Liva},
  journal={Quantum Machine Intelligence},
  volume={2},
  number={1},
  pages={11},
  year={2020},
  publisher={Springer}
}

@inproceedings{lumbreras2024linear,
  title={Linear bandits with polylogarithmic minimax regret},
  author={Lumbreras, Josep and Tomamichel, Marco},
  booktitle={The Thirty Seventh Annual Conference on Learning Theory},
  pages={3644--3682},
  year={2024},
  organization={PMLR}
}

@inproceedings{chu2011contextual,
  title={Contextual bandits with linear payoff functions},
  author={Chu, Wei and Li, Lihong and Reyzin, Lev and Schapire, Robert},
  booktitle={Proceedings of the fourteenth international conference on artificial intelligence and statistics},
  pages={208--214},
  year={2011},
  organization={JMLR Workshop and Conference Proceedings}
}

@article{nagaj2009fast,
  author  = {Nagaj, Daniel and Wocjan, Pawel and Zhang, Yong},
  title   = {Fast amplification of {QMA}},
  journal = {Quantum Information and Computation},
  year    = {2009},
  volume  = {9},
  pages   = {1053--1068}
}

@article{brassard2000quantum,
  title={Quantum amplitude amplification and estimation},
  author={Brassard, Gilles and Hoyer, Peter and Mosca, Michele and Tapp, Alain},
  journal={arXiv preprint quant-ph/0005055},
  year={2000}
}

@article{soare2014best,
  title={Best-arm identification in linear bandits},
  author={Soare, Marta and Lazaric, Alessandro and Munos, R{\'e}mi},
  journal={Advances in neural information processing systems},
  volume={27},
  year={2014}
}

@article{fiez2019sequential,
  title={Sequential experimental design for transductive linear bandits},
  author={Fiez, Tanner and Jain, Lalit and Jamieson, Kevin G and Ratliff, Lillian},
  journal={Advances in neural information processing systems},
  volume={32},
  year={2019}
}

@article{lumbreras2026learning,
  title={Learning pure quantum states almost without regret},
  author={Lumbreras, Josep and Terekhov, Mikhail and Tomamichel, Marco},
  journal={npj Quantum Information},
  year={2026},
  publisher={Nature Publishing Group UK London}
}

@InProceedings{huang2026balancing,
  title = 	 {Balancing Expressivity and Learnability in Quantum Kernel Bandit Optimization},
  author =       {Huang, Yuqi and Tan, Vincent Y. F. and Jose, Sharu Theresa},
  booktitle = 	 {Proceedings of the 42nd Conference on Uncertainty in Artificial Intelligence},
  pages = 	 {2274--2313},
  year = 	 {2026},
  volume = 	 {337},
  month = 	 {17--21 Aug},
}
\bibliographystyle{abbrvnat}

\end{document}